\PassOptionsToPackage{cleveref}{jmlrutils}
\PassOptionsToPackage{capitalize,noabbrev}{cleveref}

\documentclass[final,12pt]{colt2026} 

\usepackage{enumitem}

\usepackage{mathtools}
\usepackage{bm}

\newcommand{\subcap}[1]{\protect\parbox[t]{0.2\textwidth}{\centering #1}}
\newcommand{\subcaps}[1]{\protect\parbox[t]{0.25\textwidth}{\centering #1}}

\newenvironment{proofsketch}
  {\renewcommand{\proofname}{Proof sketch}\begin{proof}}
  {\end{proof}}

\usepackage{graphicx}
\usepackage{subcaption}

\title[Expectation Error Bounds for Transfer Learning in Linear Regression and Linear NNs]{Expectation Error Bounds for Transfer Learning\\in Linear Regression and Linear Neural Networks}
\usepackage{times}

\coltauthor{%
 \Name{Meitong Liu} \Email{meitong4@illinois.edu}\\
 \addr University of Illinois Urbana-Champaign
 \AND
 \Name{Christopher Jung} \Email{chrisjung@meta.com}\\
 \addr Meta
 \AND
 \Name{Rui Li} \Email{ruili23@meta.com}\\
 \addr Meta
 \AND
 \Name{Xue Feng} \Email{xuefeng.cs@gmail.com}\\
 \addr Meta
 \AND
 \Name{Han Zhao} \Email{hanzhao@illinois.edu}\\
 \addr University of Illinois Urbana-Champaign%
}

\begin{document}

\maketitle

\begin{abstract}%
  In transfer learning, the learner leverages auxiliary data to improve generalization on a main task. However, the precise theoretical understanding of when and how auxiliary data help remains incomplete. We provide new insights on this issue in two canonical linear settings: ordinary least squares regression and under-parameterized linear neural networks. For linear regression, we derive exact closed-form expressions for the expected generalization error with bias-variance decomposition, yielding necessary and sufficient conditions for auxiliary tasks to improve generalization on the main task. We also derive globally optimal task weights as outputs of solvable optimization programs, with consistency guarantees for empirical estimates. For linear neural networks with shared representations of width $q \leq K$, where $K$ is the number of auxiliary tasks, we derive a non-asymptotic expectation bound on the generalization error, yielding the first non-vacuous sufficient condition for beneficial auxiliary learning in this setting, as well as principled directions for task weight curation. We achieve this by proving a new column-wise low-rank perturbation bound for random matrices, which improves upon existing bounds by preserving fine-grained column structures. Our results are verified on synthetic data simulated with controlled parameters.
\end{abstract}


\section{Introduction}

Transfer learning, which aims to improve performance in a target domain by transferring knowledge from a source domain, has been a critical topic in machine learning~\citep{pan2009survey,kouw2018introduction,zhuang2020comprehensive}. In a similar spirit, multi-task learning (MTL) aims to boost generalization across multiple tasks by pooling shared information through joint training~\citep{caruana1997multitask,ruder2017overview,zhang2021survey}. While MTL is usually evaluated by an aggregation of task objectives, such as the weighted average, it is essential to examine whether each task actually benefits from the information transfer. Specifically, viewing a task of interest as the main task and others as the auxiliary, is the jointly trained model better on the main task than a model trained without auxiliary data? In this work, we present new theoretical results on this question.

Formally, given a training set $D = (X, Y_m, \{Y_k\}_{k=1}^K)$, where $X$ is the shared feature, $Y_m, Y_{1 \sim K}$ are the main and auxiliary labels, an estimator $\hat w_D$ is learned with auxiliary tasks weighted by $\{\lambda_k\}_{k=1}^K$. Our goal is to characterize the expected generalization error of $\hat w_D$ on the main task:
\begin{align*}
    E_K := \mathbb E_{D, (x,y_m)} \left[ (x^\top\hat w_D - y_m)^2\right],
\end{align*}
where $(x,y_m)$ is a test point drawn from the main task distribution. This enables us to investigate the following questions: (1) What are the factors that influence $E_K$, and how? (2) When do auxiliary data help, i.e, $E_K < E_0$, where $E_0$ is the error when no auxiliary data are involved? and (3) How to select the task weights $\{\lambda_k\}_{k=1}^K$ to minimize $E_K$?

We focus on two canonical linear frameworks for learning $\hat w_D$: (1) ordinary least squares regression and (2) linear neural networks with a shared representation layer of width $q$ and task-specific output heads. Even in these simple setups, existing results for the above questions remain incomplete. Prior work on linear regression often limits its analyses to $K=1$ auxiliary source, nor have they developed practical methods for selecting optimal task weights~\citep{dhifallah2021phase,dar2022double}. Linear neural networks with width $q < K+1$ pose fundamental challenges due to a random low-rank truncation in the learned estimator, leaving non-asymptotic modeling of $E_K$ an open problem~\citep{wu2020understanding,yang2025precise}. In this work, we fill in these gaps with exact or upper bounds for $E_K$, featuring the first theoretical analysis for optimal weight selection in linear regression and the first non-vacuous sufficient condition for beneficial auxiliary learning in under-parameterized linear networks, along with other insights. Technically, we also contribute to the study of random matrix theory with a new column-wise low-rank perturbation bound.

\subsection{Our Contributions}

\paragraph{Linear regression} Our technical results can be summarized as follows:
\begin{itemize}[topsep=0.0em, itemsep=0.0em, leftmargin=1.0em]
    \item We derive exact expressions for $E_K$ with bias-variance decomposition (\cref{th:3.1}), formalizing the trade-off mechanism that auxiliary data increases bias but potentially decreases variance.
    \item We obtain necessary and sufficient conditions for beneficial auxiliary learning as an upper bound on the total weight of auxiliary tasks (\cref{co:3.2}), which translates to a call for higher task similarity, smaller composite auxiliary noise, and moderate sample size for positive transfer.
    \item We solve for the optimal task weights that minimize $E_K$ via a convex quadratic program over the simplex and a closed-form formula (\cref{th:3.2}). Based on empirical estimates for the true solutions, we develop a practical method for weight tuning with consistency guarantees (\cref{th:3.3}), offering the first weight optimization strategy for transfer learning in linear regression.
\end{itemize}

\paragraph{Linear neural networks} We focus on the under-parameterized case where the network width $q < K+1$, as otherwise the model capacity is too large for any information transfer~\citep{wu2020understanding}. In this case, the central difficulty arises from the highly nonlinear low-rank approximation of a random matrix in the learned estimator, rendering an exact characterization of $E_K$ intractable. We instead pursue an upper bound $U(E_K)$ with the critical requirement that it can be less than $E_0$, so that we obtain useful information by setting $U(E_K)<E_0$. By proving a new column-wise low-rank perturbation bound and careful concentration of random matrices, we derive such a $U(E_K)$:

\begin{theorem}[informal; see \cref{th:4.1}]
    Let $N$ be the number of samples and $d$ be the feature dimension with $N > d + 3$. Define the weight matrix $\Lambda = \operatorname{diag}(\{\sqrt{\lambda_k}\}_{k=1}^K \cup \{1\})$ and the task matrix $W^* = [w^*_1, \ldots, w^*_K, w^*_m]$, where $w^*_m$ (resp. $w^*_k$) is the true model for the main task (resp. the $k$-th auxiliary task). Given that a signal-to-noise ratio $r$ tunable by task weights satisfy $r > 1$, we have
    \begin{align*}
    E_K \leq \mathcal O(1) \sigma_{q+1}^2(W^*\Lambda) + \mathcal O(1) \left( \frac{\|w^*_m\|}{r-1}\right)^2 + \mathcal O\left(\frac{1}{N}\right) \left( \frac{d}{(r-1)^2} + q\right) =: U(E_K), 
    \end{align*}
    where $\mathcal O(\cdot)$ reflects the orders of $N$ and $d$, and $\sigma_{q+1}(\cdot)$ denotes the $(q+1)$-th largest singular value.
\end{theorem}

\begin{table}[t]
    \caption{\small Comparison of low-rank perturbation bounds. Here, $\delta_q(S) = \sigma_q(S) - \sigma_{q+1}(S)$, $r = \delta_q(S) / \|Z\|$, and $P_U = U_{:,q}U_{:,q}^\top$, where $U_{:,q}$ contains the top $q$ left singular vectors of $S$.}
    \label{tab:1}
    \centering
    \resizebox{\textwidth}{!}{
    {\renewcommand{\arraystretch}{1.3}
    \begin{tabular}{ccccc}
    \hline
    Source                       &          Object          & Assumption & Bound & \begin{tabular}[c]{@{}l@{}}Possible to be smaller\\ than $\|Z\|$ or $\|Ze_j\|$?\end{tabular}\\ \hline
    EYM~\citep{tran2025spectral} &   $\|[M]_q - [S]_q\|$    &    -       &   $2(\sigma_{q+1}(S) + \|Z\|)$    &   No  \\
    TVV~\citep{tran2025spectral} &   $\|[M]_q - [S]_q\|$    &   $r\geq4$ &   $\frac{6r}{r-2}\|Z\|(\frac{\sigma_q(S)}{\delta_q(S)} + \log \frac{6\sigma_1(S)}{\delta_q(S)})$    &  No  \\
    Ours (\cref{th:4.2})         & $\|([M]_q - [S]_q)e_j\|$ &  $r > 1$  &   $\frac{1}{r-1}(\|Se_j\|+\|Ze_j\|) + \|P_UZe_j\|$   &  Yes\\ \hline
    \end{tabular}
    }
    }
\end{table}

Our bound reveals two driving factors: low-rank truncation loss as $\sigma_{q+1}(W^*\Lambda)$ and a tunable signal-to-noise ratio $r$. Crucially, by curating task weights to reduce $\sigma_{q+1}(W^*\Lambda)$ and increase $r$, we can have $E_K < U(E_K) < E_0$. This establishes the first non-vacuous sufficient condition for beneficial auxiliary learning in this setting. By the interaction between $\Lambda$ and other parameters in $r$ and $\sigma_{q+1}(W^*\Lambda)$, the bound also provides principled suggestions on weight selection.

The proof of $U(E_K)$ relies on our new column-wise low-rank perturbation bound (\cref{th:4.2}). Given a signal matrix $S$ and a subGaussian perturbation $Z$, with $M = S + Z$, we aim to bound $\|([M]_q-[S]_q)e_j\|$, where $e_j$ is the $j$-th standard basis. In order to have $U(E_K) < E_0$, the bound must not be trivially larger than the original column perturbation $\|Ze_j\|$. However, naively relaxing $\|([M]_q - [S]_q)e_j\| \leq \|[M]_q - [S]_q\|$ and applying state-of-the-art results~\citep{tran2025spectral} fails to achieve this: as in \cref{tab:1}, the coefficients for $\|Z\|$ in these bounds are always larger than one. In contrast, our bound remains nontrivial: as it can be proven that $\|P_UZe_j\|$ is asymptotically smaller than $\|Ze_j\|$, with a sufficiently large $r$, the bound improves upon $\|Ze_j\|$. This is achieved by preserving the fine-grained column information $e_j$, instead of coarsening to the full error. 

\subsection{Related Work}

Our work contributes to the theoretical understanding of transfer learning. Early work in domain adaptation explored generalization bounds in terms of various network complexity measures~\citep{crammer2008learning,mansour2009domain,ben2010theory}, which are often too loose to account for the positive and negative transfers. \citet{hanneke2019value,mousavi2020minimax} established minimax lower bounds to quantify the limit of how much benefit auxiliary data can provide. By contrast, our goal is to study the precise error for specific model classes and algorithms.

In the context of linear regression, \citet{dhifallah2021phase} conducted an asymptotic analysis of the target test error, where the sample size $N$ and feature dimension $d$ approach infinity under a fixed ratio. \citet{yang2025precise} further studied different distribution shifts between the source and target in this setting. \citet{dar2022double,ju2023generalization} studied the non-asymptotic behavior of over-parameterized regressors, where a subset of parameters can be re-trained on the target. On this line, we contribute our non-asymptotic analysis of $E_K$ that applies to multiple tasks rather than two as in most previous work, and the first theoretical analysis for optimizing $E_K$ over task weights.

Studies on linear neural networks remain limited. \citet{wu2020understanding} first examined the architecture employed in this work, revealing the necessity of restricting model width for information transfer. In the under-parameterized case, they bound the relative distance between the fitted and true estimators, which, however, does not indicate whether incorporating auxiliary data is beneficial. More related, \citet{yang2025precise} derived an asymptotic limit of the main task generalization error, from which our bound in \cref{th:4.1} differs in two ways. First, their analysis operates in the asymptotic setting, while ours applies to finite $N$ and $d$. Second, given the limiting scenario, their proof relies on high-probability events without tail specification. In contrast, we provide an expectation bound that accounts for the full distribution. Given such distinctions, the two results are generally incomparable and offer different perspectives on this rarely studied setting.

For the low-rank perturbation, apart from the results in \cref{tab:1}, \citet{mangoubi2022re,mangoubi2025private} bounded the Frobenius norm $\|[M]_q-[S]_q\|_F$, which are looser by nature and do not apply to our context. To our best knowledge, no work has addressed the column-specific error.

\section{Problem Setup and Preliminaries} \label{sec:2}

\paragraph{Notation.} Given a positive integer $n$, we use $[n]$ for the set $\{1,2,\ldots,n\}$. We use $\| \cdot \|$ for the vector $\ell_2$ norm and the matrix spectral norm, and use $\|\cdot\|_F$ for the matrix Frobenius norm. We use $e_j$ to denote the standard basis vector with 1 on the $j$-th entry and 0 elsewhere. Given random variables $X$ and $Y$, $\mathbb E_X[XY]$ is the expectation over $X$.

\paragraph{Data.} We consider regression problems with one main task and $K \geq 1$ auxiliary tasks. A training set is given as $D=(X, Y_m, Y_1, \ldots, Y_K)$, where $X \in \mathbb R^{N \times d}$ is the shared feature with full column rank, $Y_m \in \mathbb R^N$ is the main task labels, and $Y_k \in \mathbb R^N, k \in [K]$ are the auxiliary labels. We assume the number of samples $N$ and the feature dimension $d$ satisfy $N > d + 1$; for linear NNs, we ask for $N > d + 3$. We also assume $K + 1 \ll d$, which holds in most real-world low-dimension scenarios.

We consider linear data modeling: for any task $t \in [K] \cup \{m\}$, labels are generated as $Y_t = Xw^*_t + \epsilon_t$, where $w^*_t \in \mathbb R^d$ is the ground-truth model and $\epsilon_t \in \mathbb R^N$ is the label noise. Following prior works \citep{wu2020understanding,mousavi2020minimax}, we assume each row of $X$ is drawn independently from $\mathcal N(0, \Sigma_x)$, where $\Sigma_x \in \mathbb R^{d \times d}$ is the covariance matrix. We assume $\|\Sigma_x\| = 1$ and $\sigma_{\min}(\Sigma_x) > 0$. Each entry of $\epsilon_t$ is drawn independently from $\mathcal N(0, \sigma_t^2)$.

\paragraph{Linear regression.} For the first part of this work, we learn a shared estimator $\hat w_D \in \mathbb R^d$ that predicts $\hat Y = X\hat w_D$ for all tasks and minimizes a weighted sum of MSE losses:
\begin{align}
\hat w_D = \underset{w \in \mathbb R^d}{\arg \min}\ \|Xw - Y_m\|^2 + \sum_{k=1}^K \lambda_k \|Xw-Y_k\|^2 , \label{eq:2.1}
\end{align}
where $\lambda_k \geq 0$ for all $k \in [K]$ are the task weights. When the main task is weighted by $\lambda_m \geq 0$ and $\lambda_m + \sum_{k=1}^K \lambda_k = 1$, this is the linear scalarization technique in multi-task learning~\citep{caruana1997multitask}. In this work, we fix $\lambda_m=1$ and only require non-negativity for the others. Our results apply to both settings with minor tweaks. Standard analysis solves \cref{eq:2.1}:
\begin{align}
	\hat w_D = \frac{1}{1+\sum_{k=1}^K \lambda_k} \left( (X^\top X)^{-1}X^\top Y_m + \sum_{k=1}^K \lambda_k (X^\top X)^{-1}X^\top Y_k\right) , \label{eq:2.2}
\end{align}
which is a convex combination of the ordinary least squares (OLS) estimators for each task.

\paragraph{Linear neural networks.} Without loss of generality, we consider two-layer neural networks with a shared layer $B \in \mathbb R^{d \times q}$ and task-specific heads $a_m \in \mathbb R^{q}$, $a_k \in \mathbb R^{q}$, $k \in [K]$. The prediction for the main task (resp. the $k$-th auxiliary task) is $\hat Y_m = XBa_m$ (resp. $\hat Y_k = XBa_k$). Any deeper linear network can be reduced to this structure of a shared encoder and separate heads. 

We adopt a weighted MSE loss similar to \cref{eq:2.1}. Define $A = [a_1, \ldots, a_K, a_m]$, $\Lambda = \operatorname{diag}(\{\sqrt{\lambda_k}\}_{k=1}^K \cup \{1\})$, and $Y = [Y_1, \ldots, Y_K, Y_m]$, the objective can be written as $\hat B_D, \hat A_D = \underset{B, A}{\arg \min}\ \|(XBA - Y)\Lambda \|_F^2$. Denoting $\hat w_m = \hat B_D \hat a_m$ and $\hat w_k = \hat B_D \hat a_k$, $k \in [K]$ as the composite estimators for each task and their stack $\hat W_D = [\hat w_1, \ldots, \hat w_K, \hat w_m]$, we can rewrite the model as
\begin{align}
\hat W_D = \underset{W \in \mathbb R^{d \times (K+1)},\ \operatorname{rank}(W) \leq q}{\arg \min}\ \|(XW-Y)\Lambda\|_F^2 . \label{eq:2.3}
\end{align}

There are two regimes based on the width of the shared layer, $q$. When $q \geq K+1$, we are in the over-parameterized case, where we have the following result:

\begin{proposition}[restated from \cite{wu2020understanding}]
    When $q \geq K+1$, any optimal solution $\hat W_D$ to \cref{eq:2.3} satisfies $\hat w_m = (X^\top X)^{-1}X^\top Y_m$, $\hat w_k = (X^\top X)^{-1}X^\top Y_k$, $k \in [K]$.
\end{proposition}

In other words, when the model is large enough, the optimal composite estimator for each task is its OLS estimator, which is equivalent to learning each task individually. This leads to the conclusion in \cite{wu2020understanding} that limiting model capacity is necessary for information transfer. Consequently, we focus on the under-parameterized case with $q < K+1$.

\begin{proposition}[restated from \cite{hu2023revisiting}] \label{pp:2.2}
    When $q < K+1$, \cref{eq:2.3} gives
    \begin{align}
    \hat W_D = X^{\dagger}[P_XY\Lambda]_q \Lambda^{-1}, \label{eq:2.4}
    \end{align}
    where $P_X = X(X^\top X)^{-1}X^\top $ is the projection matrix that maps to the column space of $X$. The optimal composite estimator for each task is then $\hat w_m = \hat W_D e_{K+1}$, $\hat w_k = \hat W_D e_k$, $k \in [K]$.
\end{proposition}

Here, $[\cdot]_q$ denotes the best rank-$q$ approximation of a matrix under any unitarily invariant norm. For a matrix $A \in \mathbb R^{m \times n}$ with SVD $A = \sum_{i=1}^{\min(m,n)} \sigma_iu_iv_i^\top $, where $\sigma_1 \geq \ldots \geq \sigma_{\min(m,n)}$, by the Eckart-Young-Mirsky Theorem~\citep{eckart1936approximation}, we have $[A]_q = \sum_{i=1}^q \sigma_i(A)u_iv_i^\top $. 

\paragraph{Learning goal.} Given the empirical risk minimizers for the two settings defined in \cref{eq:2.1,eq:2.3}, we aim for a smaller expected generalization error on the main task, namely 
\begin{align}
	E := \mathbb E_{D, (x,y_m)}\left[(x^\top \hat w_m - y_m)^2\right], \label{eq:2.5}
\end{align}
where $D$ denotes the training set and $(x,y_m)$ is freshly sampled from the main task distribution: draw $x \sim \mathcal N(0, \Sigma_x)$, $\epsilon \sim \mathcal N(0, \sigma_m^2)$, and set $y_m = x^\top w_m^* + \epsilon$. Note that $\hat w_m$ ($\hat w_D$ in linear regression) is a function of the random training set $D$.

\section{Auxiliary Data in Linear Regression} \label{sec:3}

In this section, we study when and how auxiliary data improve generalization on the main task in linear regression. In \cref{sec:3.1}, we obtain the exact expected generalization errors, showing that auxiliary learning can be understood as a bias-variance trade-off. In \cref{sec:3.2}, we derive the necessary and sufficient conditions for beneficial incorporation and the optimal task weights. These results formalize the interaction between task weights, task similarity, noise levels, and sample size, and offer a practical method for weight selection with good statistical properties.

\subsection{From the View of Bias-Variance Trade-off} \label{sec:3.1}

Starting from \cref{eq:2.5}, we first apply the bias-variance decomposition. This step is not a must, but it helps interpret the effect of considering auxiliary tasks. Specifically, we have
\begin{align*}
E = \mathbb E_{D,\ (x,y_m)} \left[ (x^\top  \hat w_D - y_m)^2\right]
= \underbrace{\mathbb E_{D,x}\left[ (x^\top (\hat w_D - \bar w))^2\right]}_{\text{variance}} + \underbrace{\mathbb E_x \left[ (x^\top (\bar w - w^*_m))^2\right]}_{\text{bias}} + \underbrace{\sigma^2_m}_{\text{noise}}
\end{align*}
where $\bar w = \mathbb E_{D} [ \hat w_D ]$ is the expected estimator. Conceptually, variance arises from the randomness of a finite training set. As the sample size grows to infinity, it vanishes to 0. Bias is determined by the learning algorithm and the capacity of the hypothesis class. It does not scale with the sample size. The noise term is $\mathbb E_{(x,y_m)}[(x^\top w^*_m - y_m)^2] = \mathbb E_{\epsilon}[ \epsilon^2]$, which comes from the label noise.

Following this framework, we obtain the exact form of the expected generalization error:

\begin{theorem} \label{th:3.1}
    Let $E_K$ be the expected main task generalization error when $K$ auxiliary tasks are included. Suppose the auxiliary tasks are weighted by $\{\lambda_k\}_{k=1}^K$, then we have
    \begin{align}
    E_K = \underbrace{\frac{\sigma_m^2 + \sum_{k=1}^K \lambda_k^2 \sigma_k^2}{(1+\Lambda)^2} \frac{d}{N-d-1}}_{\text{variance}} + \underbrace{\left(\frac{\Lambda}{1 + \Lambda}\right)^2 \| \Sigma_x^{\frac{1}{2}}(w^*_{\textrm{aux}} - w^*_m)\|^2}_{\text{bias}} + \underbrace{\sigma_m^2}_{\text{noise}} , \label{eq:3.3}
    \end{align}
    where $\Lambda = \sum_{k=1}^K \lambda_k > 0$ and $w^*_{\textrm{aux}} = \frac{1}{\Lambda}\sum_{k=1}^K \lambda_k w^*_k$. In particular, when $K=0$ or $\Lambda = 0$, the expected generalization error of learning the main task only is
    \begin{align}
    E_0 = \underbrace{\sigma_m^2 \frac{d}{N-d-1}}_{\text{variance}} + \underbrace{0}_{\text{bias}} + \underbrace{\sigma_m^2}_{\text{noise}} . \label{eq:3.2}
    \end{align}
\end{theorem}
\begin{proofsketch}
    The bias follows from computing $\bar w$ from \cref{eq:2.2}. The variance follows from the trace-expectation identity and the mean of an inverse-Wishart distribution. See \cref{sec:ap-a1}.
\end{proofsketch}

From the decompositions of $E_0$ and $E_K$, it is clear that bringing in auxiliary information increases bias due to their difference from the main task. Indeed, the bias term in $E_K$ evaluates a weighted distance between $w^*_{\textrm{aux}}$ and $w^*_m$. On the other hand, variance can be reduced given proper weights and auxiliary label noise. For example, when $\sigma_k^2 = \sigma_m^2$ and $\lambda_k = 1$ for all $k \in [K]$, the noise coefficient is reduced by a factor of $K+1$. Therefore, to benefit from the auxiliary data, one must ensure that the decrease in variance outweighs the increase in bias.

\subsection{Utility Conditions and Optimal Weighting} \label{sec:3.2}

Given \cref{th:3.1}, we obtain the necessary and sufficient conditions for auxiliary tasks to help, i.e., $E_K < E_0$, as well as the optimal task weights that minimize $E_K$. 

We first introduce a necessary re-parameterization. When $K>1$, directly optimizing \cref{eq:3.3} over $\{\lambda_k\}$ yields a nonconvex fractional quadratic program, which admits no closed-form solution or solver with global optimality guarantees. To resolve this, we re-parameterize $\{\lambda_k\}$: let $\Lambda = \sum_{k=1}^K \lambda_k > 0$, and for any $k \in [K]$, let $\lambda_k' = \lambda_k / \Lambda$, such that $\sum_{k=1}^K \lambda_k' = 1$, i.e., $\bm \lambda' = [\lambda_1', \ldots, \lambda_K'] \in \Delta_K$, where $\Delta_K$ denotes the $(K-1)$-dimensional simplex. Here, $\Lambda$ controls the overall scale of the auxiliary tasks, and $\{\lambda_k'\}$ controls their proportion. The optimal values are then solvable. We thus present results in this subsection using the new parameterization.

The following corollary provides the necessary and sufficient conditions for beneficial auxiliary learning. It is a direct result of setting $E_K < E_0$ using \cref{th:3.1} and some reorganization.

\begin{corollary} \label{co:3.2}
    Suppose the auxiliary tasks are weighted by $\{\lambda_k\}_{k=1}^K$, then $E_K < E_0$ iff.
    \begin{align}
    \underbrace{\left( (\sum_{k=1}^K (\lambda_k')^2 \sigma_k^2 - \sigma_m^2) \frac{d}{N-d-1} + \| \Sigma_x^{\frac{1}{2}}(w^*_{\textrm{aux}} - w^*_m)\|^2 \right)}_{L}\Lambda < \underbrace{2\sigma_m^2 \frac{d}{N-d-1}}_{R} , \label{eq:3.6}
    \end{align}
    where $\Lambda = \sum_{k=1}^K \lambda_k$, $\lambda_k' = \lambda_k / \Lambda$, $k \in [K]$, and $w^*_{\textrm{aux}} = \frac{1}{\Lambda}\sum_{k=1}^K \lambda_k w^*_k = \sum_{k=1}^K \lambda_k' w^*_k$.
\end{corollary}

\paragraph{Remark} Recall that $\Lambda$ denotes the total strength of the auxiliary tasks. Hence, by \cref{eq:3.6}, for a meaningful incorporation, $L$ should be either (1) negative, and any weight combination with $\Lambda > 0$ suffices, or (2) positive but small, so that after dividing it on both sides, the resulting upper limit on $\Lambda$, i.e., $\Lambda < \frac{R}{L}$ is not trivially tiny. This translates to the needs of several factors: 
\begin{itemize}[topsep=0.0em, itemsep=0.0em]
    \item \textbf{Higher task similarity} in terms of a smaller distance term, $\| \Sigma_x^{\frac{1}{2}}(w^*_{\textrm{aux}} - w^*_m)\|$.
    \item \textbf{Smaller composite auxiliary label noise}, $\sum_{k=1}^K (\lambda_k')^2\sigma_k^2$.
    \item \textbf{Moderate sample size}: With a fixed $d$, as $N$ grows to infinity, the LHS coefficient becomes positive, while the RHS vanishes, which forces $\Lambda$ to be near 0. Hence, for auxiliary data to be influential, the shared samples should not be ``too sufficient''. Otherwise, as discussed in \cref{sec:3.1}, there is not much variance to be traded with bias. 
\end{itemize}

In addition to the inherent task structures, one can also tune $\bm \lambda'$ to optimize these factors. In fact, the optimal weights are solvable after the re-parameterization.

\begin{theorem} \label{th:3.2}
    For any $\bm \alpha \in \Delta_K$, define $f(\bm \alpha) = \frac{d}{N-d-1} \sum_{k=1}^K (\alpha_k)^2 \sigma_k^2 + \delta(\bm \alpha)^\top \Sigma_x\delta(\bm \alpha)$, where $\delta(\bm \alpha) = \sum_{k=1}^K \alpha_k w^*_k - w^*_m$. The optimal weights that minimize $E_K$ as in \cref{eq:3.3} are
    \begin{align}
    	\bm \lambda'^* = \underset{\bm \alpha \in \Delta_K}{\arg \min}\ f(\bm \alpha), \qquad
    	\Lambda^* = \sigma_m^2\frac{d}{N-d-1} \cdot \frac{1}{f(\bm \lambda'^*)}, \label{eq:3.8}
    \end{align}
    and we recover $\lambda_k^* = \Lambda^* \lambda_k'^*$, $k \in [K]$. Optimizing $f(\bm \alpha)$ is a convex quadratic program over the simplex, which can be solved by off-the-shelf methods. When $K=1$, we have $\lambda'^* = 1$,$\lambda^* = \Lambda^*$.
\end{theorem}
\begin{proof}
    By \cref{eq:3.3}, minimizing $E_K$ is equivalent to minimizing $\frac{c + \Lambda^2f(\bm \lambda')}{(1+\Lambda)^2}$, where $c$ is a constant. Optimizations for $\Lambda$ and $\bm \lambda'$ can then be disentangled: for $\bm \lambda'$, it reduces to minimizing $f(\bm \lambda')$; for a fixed $\bm \lambda'$, the optimal $\Lambda^*$ is obtained by the first- and second-order optimality conditions.
\end{proof}

Based on \cref{th:3.2}, we provide guidance on selecting task weights given a training set in practice, where the distribution parameters are unknown. As a standard approach, we replace unknowns with data-derived estimators and solve for the estimated optimal weights, denoted as $\hat{\bm \lambda}^*$. Let us then consider its statistical properties. Since $E_K$ contains an interaction term between $w^*$ and $\Sigma_x$, it is difficult to ensure unbiasedness. Optimizing an unbiased estimator for $E_K$ also does not guarantee an unbiased estimator for the true optimal weights. Nevertheless, $\hat{\bm \lambda}^*$ proves to be consistent:

\begin{theorem} \label{th:3.3}
    Given a training set $D = (X, Y_m, Y_1, \ldots, Y_k)$, for $t \in [K] \cup \{m\}$, we have $\hat w_t = (X^\top X)^{-1}X^\top Y_t\overset{p}{\rightarrow} w^*_t$, $\hat \sigma_t^2 = \frac{\|Y_t - X\hat w_t\|^2}{N - d}\overset{p}{\rightarrow} \sigma_t^2$, and $\hat \Sigma_x = \frac{1}{N}X^\top X \overset{p}{\rightarrow} \Sigma_x$. Let $\hat{\bm \lambda}$ be the solution to \cref{eq:3.8} after plugging in these consistent estimators, and $\bm \lambda^* = \underset{\bm \lambda}{\arg \min}\ E_K$ be the true optimal weights. If a unique $\bm \lambda^*$ exists, and $\exists B < \infty$ such that $\|\bm \lambda^*\| \leq B$, $\|\hat{\bm \lambda}\| \leq B$, then $\hat{\bm \lambda} \overset{p}{\rightarrow} \bm \lambda^*$.
\end{theorem}
\begin{proofsketch}
    This follows from the consistency of M-estimators. See \cref{sec:ap-a2}.
\end{proofsketch} 

Since the given consistent estimators are asymptotically Gaussian, one can also approximate the variance and confidence intervals for $\hat{\bm \lambda}$ by its asymptotic normality, which can be justified via either the delta method~\citep{cramer1999mathematical} or the properties of M-estimators~\citep{van2000asymptotic}. 

Conventionally, tuning task weights requires a grid search over a held-out set. For large $K$, this can be computationally and statistically expensive. Our theoretical analysis and practical solution offer a way to identify a provably good candidate, near which a small-scale grid search can be conducted for finer tuning. Finally, we note that although the above method is based on a given $D$, it does not aim to minimize the error for this specific set. Instead, by the definition in \cref{eq:2.5}, we aim to be optimal in expectation over randomly drawn sets.

\section{Auxiliary Data in Under-Parameterized Linear Networks} \label{sec:4}

Next, we study the setting of under-parameterized linear neural networks. In \cref{sec:4.1}, we present the expectation bound for the main task generalization error, integrating task weights, task structures, model capacity, and a tunable signal-to-noise ratio. The bound is informative in the sense of yielding a non-vacuous sufficient condition for beneficial auxiliary learning. We also introduce our column-wise low-rank perturbation bound, which possesses key potential not achieved by existing results to improve upon the original perturbation. We then walk through the proofs in \cref{sec:4.2}.

\subsection{An Informative Expectation Bound on the Generalization Error by a New Column-Wise Low-Rank Perturbation Bound} \label{sec:4.1}

To compute the expected generalization error, one has to consider the random behavior of the learned estimator $\hat w_m$. By \cref{pp:2.2}, the solution to this setting consists of a low-rank approximation of a random matrix, $[P_XY\Lambda]_q$, whose exact expectation is generally hard to derive due to the highly nonlinear SVD operation. Hence, we aim for an upper bound. 

\begin{theorem} \label{th:4.1}
    Let $E_K$ denote the expected main task generalization error when $K$ auxiliary tasks are included with weights $\{\lambda_k\}_{k=1}^K$ via a linear network with width $q < K + 1$. Define $\Lambda = \operatorname{diag}(\{\sqrt{\lambda_k}\}_{k=1}^K \cup \{1\})$, $W^* = [w^*_1, \ldots, w^*_K, w^*_m]$, $\epsilon = [\epsilon_1, \ldots, \epsilon_K, \epsilon_m]$, where $\epsilon_t = Y_t - Xw^*_t$, $\forall t \in [K] \cup \{m\}$, and $P_X = X(X^\top X)^{-1}X^\top $. If $\frac{\sigma_q(XW^*\Lambda) - \sigma_{q+1}(XW^*\Lambda)}{\|P_X\epsilon\Lambda\|} \geq r > 1, \forall X, \forall \epsilon$, we have
    \begin{align}
    E_K & \leq \underbrace{2 C_4 \kappa(\Sigma_x)\left( \frac{\sqrt{N} + \sqrt{d}}{\sqrt{N} -\frac{d+3}{\sqrt{N}}} \right)^2 \sigma_{q+1}^2(W^*\Lambda)}_{T_1} + \underbrace{6 C_4 \kappa(\Sigma_x)\left( \frac{\sqrt{N} + \sqrt{d}}{\sqrt{N} - \frac{d+3}{\sqrt{N}}} \right)^2 \left( \frac{\|w_m^*\|}{r-1} \right)^2}_{T_2} \nonumber \\
        & \quad \ + \underbrace{6 \sigma_m^2\frac{C_3}{\sigma_{\min}(\Sigma_x)} \left( \frac{1}{\sqrt{N} - \frac{d+1}{\sqrt{N}}} \right)^2 \left( C_1\frac{d}{(r-1)^2} + C_2q\right)}_{T_3} + \sigma_m^2 =: U(E_K), \nonumber
    \end{align}
    where $\sigma_q(\cdot)$ is the $q$-th largest singular value, $\kappa(\cdot)$ is the condition number, and $C_{1 \sim 4}$ are constants.
\end{theorem}

\paragraph{Remark}
Let us comment on the theorem. First, the condition on $r$ characterizes a strength ratio of noise-free label information to label noise. In other words, $r$ serves as a lower bound of the \textit{signal-to-noise ratio}, rendering $r > 1$ a reasonable requirement. Next, one sees that $U(E_K)$ decreases with (1) a smaller $\sigma_{q+1}(W^*\Lambda)$ in $T_1$, which reflects how much information is left out by the low-rank truncation, and (2) naturally, a larger $r$, which reduces $T_2$ and $T_3$. These values integrate the effects of task weights $\Lambda$, model capacity $q$, how aligned tasks are, and their noise levels. Given a fixed $q$, the bound suggests tuning $\Lambda$ to amplify high-quality auxiliary tasks that align with the main and suppress distinct, noisy ones, as this leads to a clustered spectrum with a larger spectral gap, a smaller composite noise, and, therefore, (1) less truncation loss and (2) a larger $r$.

Still, our goal is to reveal when auxiliary data help. With $q \geq 1$, learning the main task alone gives $\hat w_m = (X^\top X)^{-1}X^\top Y_m$. Hence, $E_0$ is the same as in \cref{eq:3.2}. By setting $U(E_K) < E_0$, one obtains a \textit{sufficient} condition for $E_K < E_0$. However, the validity of such a condition relies on a critical question: \textit{Is $U(E_K)$ possible to be smaller than $E_0$ by tuning the task weights?} If not, the bound may be too loose to provide useful information. 

Our answer is affirmative. Comparing $T_3$ with the variance of $E_0$, $\sigma_m^2 \frac{d}{N-d-1}$, both are of order $\mathcal O(d/N)$ given that $q < K+1 \ll d$. If $r$ is large, the coefficient of $d$ in $T_3$ can be significantly reduced, making $T_3$ smaller than the $E_0$ variance. When such a reduction outweighs the increase from $T_1$ and $T_2$, which can also be minimized with a larger $r$ and a smaller truncation loss, we have $U(E_K) < E_0$. By form, one can view $T_1$ and $T_2$ as the bias of order $\mathcal O(1)$, and $T_3$ as the variance of order $\mathcal O(1/N)$. Then, the condition coincides with the bias-variance trade-off rule.

In the proof of \cref{th:4.1}, one key step is to bound a column-wise low-rank perturbation: $\|([M]_q - [S]_q)e_j\|$, where $M = S + Z$ and $Z$ is a Gaussian random matrix. To make $U(E_K) < E_0$ possible, the bound must not be trivially larger than the original perturbation $\|Ze_j\|$. Unfortunately, existing state-of-the-art tools cannot achieve this goal~\citep{tran2025spectral}. We hence develop the following theorem, which preserves fine-grained column structures and harnesses a similar signal-to-noise ratio as in \cref{th:4.1}. This result might also be of independent interest to future work. 
\begin{theorem} \label{th:4.2}
    For any matrix $S, Z \in \mathbb R^{m \times n}$, let $M = S + Z$. When $\frac{\sigma_q(S) - \sigma_{q+1}(S)}{\|Z\|} \geq r > 1$, we have, for any $q \in [\operatorname{rank}(S)]$ and $j \in [n]$,
    \begin{align}
    \|([M]_q - [S]_q)e_j\| \leq \frac{1}{r-1}(\|Se_j\|+\|Ze_j\|) + \|P_UZe_j\| =: U_{q,j}, \nonumber
    \end{align}
    where $P_U$ is the projection matrix that maps onto the subspace spanned by the first $q$ left singular vectors of $S$, i.e., with SVD $S = U\Sigma V^\top $, we have $P_U = U_{:,q}U_{:,q}^\top $. 
\end{theorem}

To see how $U_{q,j}$ can be smaller than $\|Ze_j\|$, we first argue that $\|P_UZe_j\| < \|Ze_j\|$ asymptotically when $Ze_j$ is mean-zero subgaussian as in our context: due to the rank-$q$ projection, $\|P_UZe_j\|$ is of order $\mathcal O(\sqrt{q})$, while $\|Ze_j\|$ is of order $\mathcal O(\sqrt{m})$. Then, when $r$ is large and $\|Se_j\|$ is small, we can have $U_{q,j} < \|Ze_j\|$. We defer the proof and elaboration to \cref{sec:4.2.1}. 

\subsection{Proof of \cref{th:4.1}} \label{sec:4.2}

The proof of \cref{th:4.1} proceeds in two main steps: (1) bounding the column-wise low-rank perturbation (\cref{sec:4.2.1}) and (2) bounding expectations of random matrices (\cref{sec:4.2.2}).

We first introduce three key matrices $M,S,Z$. We use $e_m$ for $e_{K+1}$, as the $(K+1)$-th column holds the main task values. Let $M = P_XY\Lambda$; by \cref{eq:2.4}, we have $\hat w_m = X^{\dagger}[M]_q\Lambda^{-1} e_m = X^{\dagger}[M]_q e_m$, where the second step comes from the last entry of $\Lambda$ being $1$. Let $S = P_X XW^*\Lambda = XW^*\Lambda$ be the pure signal, and $Z = P_X \epsilon\Lambda$ be the projected noise. With $Y = XW^* + \epsilon$, we have $M = S + Z$. Finally, by the definition of $S$, we have $w^*_m = X^{\dagger}S\Lambda^{-1} e_m = X^{\dagger}Se_m$. 

The proof starts from the definition in \cref{eq:2.5}. First, one sees that bounding $E_K$ boils down to bounding $\mathbb E_D \| \hat w_m - w^*_m\|^2$. Specifically, with $y_m = x^\top w^*_m + \epsilon_m$, we have $E_K = \mathbb E_{D, (x,y_m)} [ (x^\top  \hat w_m - y_m)^2 ] = \mathbb E_{D,x} [ (x^\top  (\hat w_m - w^*_m))^2 ] + \sigma_m^2$. Let $\Delta = \hat w_m - w_m^*$; we have
\begin{align}
    \mathbb E_{D, x}\left[(x^\top \Delta)^2\right] = \mathbb E_D \left[ \Delta^\top  \mathbb E_x[xx^\top ] \Delta \right] = \mathbb E_D [\Delta^\top  \Sigma_x \Delta] \leq \|\Sigma_x\| \cdot \mathbb E_D \|\Delta\|^2  = \mathbb E_D \|\Delta\|^2 \label{eq:4.4}
\end{align}
where the inequality follows from $|\Delta^\top \Sigma_x\Delta| \leq \|\Delta\| \cdot \|\Sigma_x\| \cdot \|\Delta\|$. Next, we plug in $\hat w_m = X^{\dagger}[M]_qe_m$ and $w^*_m = X^{\dagger}Se_m$ derived above and perform an add-and-subtract operation with $[S]_q$:
\begin{align}
\mathbb E_D \|\hat w_m - w^*_m\|^2 & = \mathbb E_D \| X^{\dagger}([M]_q - [S]_q + [S]_q -S)e_m \|^2 \nonumber \\
& \leq \mathbb E_D \left[ \|X^{\dagger}\|^2 \left( 2\|([M]_q - [S]_q)e_m\|^2 + 2\|([S]_q-S)e_m\|^2\right) \right] . \label{eq:4.3}
\end{align}
For the second term, one has $\|([S]_q - S)e_m\| \leq \|[S]_q - S\| \|e_m\| = \sigma_{q+1}(S)$ by the Eckart-Young-Mirsky (EYM) Theorem~\citep{eckart1936approximation}. We are then left with $\|([M]_q - [S]_q)e_m\|$.

\subsubsection{Bounding the Column-Wise Low-Rank Perturbation} \label{sec:4.2.1}

We first illustrate why a useful bound $U_{q,m}$ for $\|([M]_q - [S]_q)e_m\|$ should not always be larger than $\|Ze_m\|$. Consider the competitor $E_0$; by the same derivation for \cref{eq:4.4}, we have $E_0 = \mathbb E_D \|\Delta_0\|^2 + \sigma_m^2$, where $\Delta_0 = \hat w_m - w^*_m = (X^\top X)^{-1}X^\top (Y - Xw^*_m) = X^{\dagger}P_X\epsilon_m$. Notice that $Ze_m = P_X \epsilon \Lambda e_m = P_X \epsilon e_m = P_X \epsilon_m$. Therefore, proceeding from \cref{eq:4.3} with a $U_{q,m} \geq \|Ze_m\|$ would lead to an uninformative $U(E_K)$ such that
\begin{align*}
    U(E_K) - \sigma_m^2 \geq \mathbb E_D \left[ \|X^{\dagger}\|^2 \|Ze_m\|^2 \right] \geq \mathbb E_D \|X^{\dagger}P_X\epsilon_m\|^2 = \mathbb E_D \|\Delta_0\|^2 = E_0 - \sigma_m^2.
\end{align*}
Thereby, we formalize a more general subproblem: for any matrix $S, Z \in \mathbb R^{m \times n}$, let $M = S + Z$; for any $q \in [\operatorname{rank}(S)]$ and $j \in [n]$, find an upper bound for $\|([M]_q - [S]_q)e_j\|$ that is not always larger than $\|Ze_j\|$. We first discuss two existing bounds that fail to satisfy this requirement, followed by the proof of our \cref{th:4.2}, which succeeds by preserving fine-grained column information.

\paragraph{Attempt 1.} By the EYM Theorem and Weyl's inequality~\citep{weyl1912asymptotische}, a direct bound is $\|([M]_q - [S]_q)e_j\| \leq \|[M]_q - [S]_q\| \cdot 1 \leq 2(\sigma_{q+1}(S) + \|Z\|)$. We refer readers to Appendix A of \cite{tran2025spectral} for its short proof. Clearly, this bound is always larger than $\|Z\|$, and thus $\|Ze_j\|$.

\paragraph{Attempt 2.} Recently, \citet{tran2025spectral} proved a new bound for $\|[M]_q - [S]_q\|$ using a contour bootstrapping method, where $S$ and $Z$ are symmetric. To apply it to our problem, we extend it to rectangular matrices and align the assumption form with ours:

\begin{corollary}[Theorem D.1 of \cite{tran2025spectral}] \label{co:4.3}
    Given that $\frac{\sigma_q(S) - \sigma_{q+1}(S)}{\|Z\|} \geq r \geq 4$, we have $\|[M]_q - [S]_q\| \leq\frac{6r}{r-2}\|Z\|(\frac{\sigma_q(S)}{\sigma_q(S) - \sigma_{q+1}(S)} + \log \frac{6\sigma_1(S)}{\sigma_q(S) - \sigma_{q+1}(S)})$.
\end{corollary}

The complete proof is in \cref{sec:ap-b1}. Since $r\geq4$ and singular values are non-negative, the coefficient of $\|Z\|$ in this bound is always larger than 1. Hence, using $\|([M]_q - [S]_q)e_j\| \leq \|[M]_q - [S]_q\|$ and applying \cref{co:4.3} cannot meet our requirement. Inspired by our final result, we also attempted to inject the column information $e_j$ into the proof framework. However, due to the form of the contour method, adding $e_j$ does not contribute to a potentially smaller coefficient\footnote{There is another contour-based bound in \cite{tran2025spectral} (Theorem D.5) that is sometimes tighter than \cref{co:4.3}, but still larger than $\|Z\|$, and the difficulty of incorporating $e_j$ persists.}.

\paragraph{Attempt 3.} The above unsuccessful attempts essentially aim to address the entire error $\|[M]_q - [S]_q\|$. They suggest the necessity of considering the role of $e_j$ as opposed to discarding it via $\|e_j\|=1$. With this in mind, we develop \cref{th:4.2}. 

{\renewcommand{\proofname}{Proof of \cref{th:4.2}}
\begin{proof}
    With SVD $M = \hat U \hat \Sigma \hat V^\top $ and $S = U\Sigma V^\top $, define projection matrices $\hat P_U = \hat U_{:,q}\hat U_{:,q}^\top $ and $P_U = U_{:,q}U_{:,q}^\top $. We have $[M]_q = \hat P_U M$ and $[S]_q = P_US$. Then, by adding and subtracting $P_UMe_j$, followed by the triangle inequality, we have
    \begin{align}
    \|([M]_q - [S]_q)e_j\| = \|(\hat P_U M - P_U S)e_j\| & \leq \|(\hat P_U - P_U)Me_j\| + \|P_U(M-S)e_j\| \nonumber \\
    & \leq \|\hat P_U - P_U\| (\|Se_j\| + \|Ze_j\|) + \|P_UZe_j\| , \label{eq:4.5}
    \end{align}
    where the last step follows from $\|(\hat P_U - P_U)Me_j\| \leq \|\hat P_U - P_U\|\|Me_j\|$ and $M = S + Z$. Next, we bound $\|\hat P_U - P_U\|$ by the Wedin $\sin \Theta$ Theorem~\citep{wedin1972perturbation}: If $\sigma_{q}(S) - \sigma_{q+1}(M) > 0$, one has $\|\hat P_U - P_U\| \leq \frac{\|Z\|}{\sigma_q(S) - \sigma_{q+1}(M)}$. To unify the singular values, we replace $\sigma_{q+1}(M)$ with $\sigma_{q+1}(M) \leq \|Z\| + \sigma_{q+1}(S)$ by Weyl's inequality. Hence, the statement translates to: if $\sigma_q(S) - \sigma_{q+1}(M) \geq \sigma_q(S) - \sigma_{q+1}(S) - \|Z\| > 0$, i.e., $\frac{\sigma_q(S) - \sigma_{q+1}(S)}{\|Z\|} \geq r > 1$, we have $\|\hat P_U - P_U\| \leq \frac{\|Z\|}{\sigma_q(S) - \sigma_{q+1}(S) -\|Z\|} \leq \frac{1}{r - 1}$. Plugging this into \cref{eq:4.5} finishes the proof.
\end{proof}}

The column information $e_j$ is preserved throughout this proof. In fact, if $\|Se_j\|$ and $\|Ze_j\|$ in \cref{eq:4.5} are coarsened to $\|S\|$ and $\|Z\|$, one will have $\|S\| \geq \sigma_q(S) - \sigma_{q+1}(S) \geq r\|Z\|$, which results in a bound with $\frac{1}{r-1}(\|S\| + \|Z\|) \geq \frac{r+1}{r-1}\|Z\| > \|Z\|$. In contrast, there is no such limitation on $\|Se_j\|$. It is this key observation of the fine-grained column structure that leads to the superiority of our bound. Finally, we substantiate our previous argument of $\|P_UZe_j\|$ being asymptotically small:

\begin{proposition} \label{lm:4.4}
    Let $Ze_j \in \mathbb R^m$ be a mean-zero subgaussian random vector with subgaussian norm $\sigma$, and $P_U$ be the rank-$q$ projection defined in \cref{th:4.2}. With probability at least $1-2e^{-t^2}$, we have $\|Ze_j\| \leq C\sigma(\sqrt{m}+t)$ and $\|P_UZe_j\| \leq C\sigma(\sqrt{q}+t)$. Alternatively, we have $\mathbb E \|Ze_j\| \leq C\sigma\sqrt{m}$ and $\mathbb E\|P_UZe_j\| \leq C\sigma\sqrt{q}$. Here, the $C$'s are constants.
\end{proposition}

In other words, projecting a random vector onto an independent low-rank subspace largely reduces its norm. Thereby, $\|P_UZe_j\|$ is nearly negligible when $q$ is small relative to $m$. This completes the claim that our bound can be smaller than $\|Ze_j\|$ when $r$ is large and $\|P_UZe_j\| < \|Ze_j\|$. The proof of \cref{lm:4.4} is deferred to \cref{sec:ap-b3}. 

\subsubsection{Bounding Expectations of Random Matrices} \label{sec:4.2.2}

We continue with the proof of \cref{th:4.1}. From \cref{eq:4.3}, we proceed by bounding $\|([M]_q - [S]_q)e_m\|$ with \cref{th:4.2}, whose condition on $r$ translates to that of \cref{th:4.1} after plugging in the context definitions of $S = XW^*\Lambda$ and $Z = P_X\epsilon\Lambda$. Next, we disentangle the random matrices, i.e., $X$ and $Z$, from constant values. With intermediate steps deferred to \cref{sec:ap-b2}, we obtain: 
\begin{align}
E_K & \leq \mathbb E_X \left[ \|X^{\dagger}\|^2 \|X\|^2 \right] \left( 2\sigma_{q+1}^2(W^*\Lambda) + 6 \left( \frac{\|w^*_m\|}{r-1} \right)^2 \right) \nonumber \\
& \quad \ + \mathbb E_X \left[ \|X^{\dagger}\|^2 \left( \mathbb E_{\epsilon | X} \|Ze_m\|^2 + \mathbb E_{\epsilon | X} \|P_UZe_j\|^2 \right) \right] + \sigma_m^2 . \label{eq:4.6}
\end{align}
One can see the correspondence between these terms and those in the final bound. It remains to address the expectations. Extending the idea of \cref{lm:4.4}, we bound those over $Z$:

\begin{lemma} \label{lm:4.5}
    Given the setups in \cref{sec:2}, where $Z = P_X\epsilon\Lambda$, we have $\mathbb E_{\epsilon|X}\|Ze_m\|^2 \leq C_1\sigma_m^2d$ and $\mathbb E_{\epsilon|X}\|P_UZe_m\|^2 \leq C_2\sigma_m^2q$, where $C_1$ and $C_2$ are constants.
\end{lemma}

The proof is provided in \cref{sec:ap-b3}. On the outside, we bound those over $X$:

\begin{lemma} \label{lm:4.6}
    Given the setups in \cref{sec:2}, we have $\mathbb E_X \|X^{\dagger}\|^2 \leq \frac{C_3}{\sigma_{\min}(\Sigma_x)} \left( \frac{1}{\sqrt{N} - (d+1)/\sqrt{N}} \right)^2$ and $\mathbb E_X [\|X^{\dagger}\|^2 \|X\|^2] \leq C_4 \kappa(\Sigma_x) \left( \frac{\sqrt{N} + \sqrt{d}}{\sqrt{N} - (d+3)/\sqrt{N}} \right)^2$, where $C_3$ and $C_4$ are constants.
\end{lemma}

The proof, deferred to \cref{sec:ap-b4}, relies on characterizing the smallest singular value of a Gaussian random matrix. Indeed, we have $\|X^{\dagger}\| = 1/\sigma_{\min}(X)$. Finally, applying \cref{lm:4.5,lm:4.6} to \cref{eq:4.6} finishes the proof of \cref{th:4.1}.

\section{Numerical Simulations}

We verify our results on synthetic data generated with controlled parameters. Detailed setups are deferred to \cref{sec:ap-c}. \cref{fig:1a} plots the linear regression setting with one auxiliary task of moderate distance to the main task and a lower noise level. One sees that the simulated and theoretical curves align well. The cleaner auxiliary task helps improve the main task generalization with small weights, but worsens it beyond a threshold, as its difference from the main task dominates. Computed by \cref{co:3.2,th:3.2}, respectively, the theoretical threshold and optimal value for the task weights prove valid, and the plug-in estimators fall in nearby regions. We also provide the figure for two auxiliary tasks in \cref{sec:ap-c}, where the first program for obtaining the optimal weights in \cref{eq:3.8} is solved by projected gradient descent~\citep{beck2017first}.

For the linear neural network setting, we first examine our column-wise low-rank perturbation bound in \cref{th:4.2}. \cref{fig:1c} plots the log-scale errors against the signal-to-noise ratio $r$. As shown, our bound tightly approximates the actual error and can improve upon the original perturbation when $r$ is large, while the TVV bound from \cref{co:4.3} is much looser due to not preserving column information. Next, \cref{fig:1d} shows the generalization errors with one auxiliary task and network width $q=1$. Here, the signal-to-noise ratio $r$ is altered with the auxiliary task weight. Computed by \cref{th:4.1}, our bound decreases below $E_0$ as $r$ grows, when the actual error is indeed better than $E_0$. This verifies the non-vacuity and validity of the sufficient condition.

\begin{figure*}[t]
  \centering
  \subfigure[\subcap{Linear regression with $K=1$.}]{%
    \includegraphics[width=0.32\textwidth]{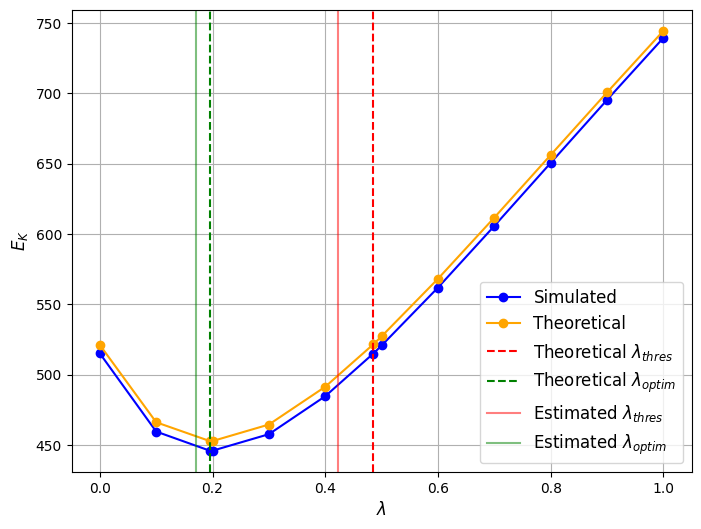}\label{fig:1a}
  }
  \subfigure[\subcaps{Column-wise low-rank perturbation bound.}]{%
    \includegraphics[width=0.32\textwidth]{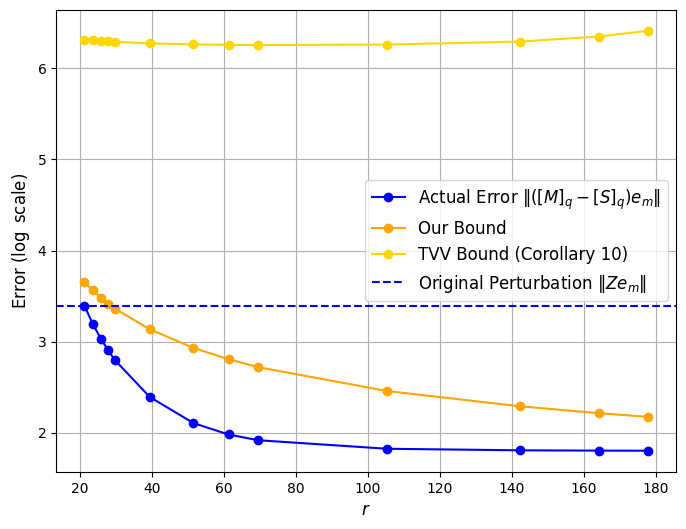}\label{fig:1c}
  }
  \subfigure[\subcaps{Generalization bound for linear NN.}]{%
    \includegraphics[width=0.32\textwidth]{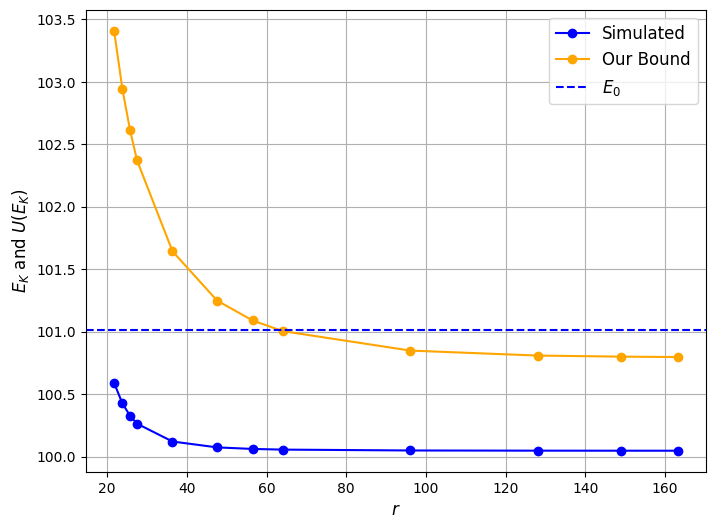}\label{fig:1d}
  }
  \caption{Simulation results of for both linear regression and linear neural networks.}
  \label{fig:1}
\end{figure*}

\section{Conclusion}

In this work, we contribute to the theoretical understanding of transfer learning by studying when and how auxiliary data improve generalization on the main task in linear regression and linear neural networks. For linear regression, we perform exact characterizations of the expected generalization error, yielding the necessary and sufficient condition for beneficial auxiliary learning. We develop globally optimal task weights that support the practical computation of consistent empirical estimators. For under-parameterized linear networks, we prove a non-asymptotic expectation bound on the generalization error in terms of a weighted signal-to-noise ratio and a tunable low-rank truncation loss. We thereby establish the first non-vacuous sufficient condition for beneficial auxiliary learning in this setting and offer principled suggestions on weight curation. Our bound builds on a new column-wise low-rank perturbation bound, which, unlike existing results, preserves the potential to improve upon the original perturbation by examining column structures.

\acks{Meitong Liu is partially supported by an Amazon AI PhD Fellowship. Meitong Liu and Han Zhao are also partially supported by an NSF CAREER Award No.\ 2442290. The authors would like to thank Yifei He, Zhe Kang, Ming Li, and Siqi Zeng for their discussion throughout the development of this work.}

\bibliography{references}

\appendix

\crefalias{section}{appendix} 

\section{Omitted Items in \cref{sec:3}}

\subsection{Proof of \cref{th:3.1}} \label{sec:ap-a1}

\begin{proof}
    We first obtain the bias term. By \cref{eq:2.2} and the data generation process, we have
    \begin{align}
        \bar w & = \frac{1}{1 + \Lambda} \left( \mathbb E_D \left[ (X^\top X)^{-1} X^\top (Xw^*_m + \epsilon_m)\right] + \sum_{k=1}^K \lambda_k\mathbb E_D \left[ (X^\top X)^{-1} X^\top (Xw^*_k + \epsilon_k) \right] \right) \nonumber \\
        & = \frac{1}{1 + \Lambda} \left( w^*_m + \mathbb E_X \left[(X^\top X)^{-1}\mathbb E_{\epsilon_m | X} [\epsilon_m] \right] + \sum_{k=1}^K \lambda_k \left(w^*_k + \mathbb E_X \left[(X^\top X)^{-1}\mathbb E_{\epsilon_k | X} [\epsilon_k] \right]\right) \right) \nonumber \\
        & = \frac{1}{1 + \Lambda} \left( w^*_m + \sum_{k=1}^K \lambda_k w^*_k\right) , \label{eq:a.1}
    \end{align}
   where the last step follows from $\mathbb E_{\epsilon_t|X}[\epsilon_t]=0$ for any $t \in [K] \cup \{m\}$. Hence, we have the bias as
   \begin{align*}
       \mathbb E_x \left[ \left(x^\top (\bar w - w^*_m)\right)^2\right] & = \mathbb E_x \left[ \left(x^\top  \frac{\Lambda(w^*_{\mathrm{aux}} - w^*_m)}{1 + \Lambda} \right)^2 \right] \\
       & = \left( \frac{1}{1+\Lambda} \right)^2 (w^*_{\textrm{aux}} - w^*_m)^\top  \mathbb E_x \left[ xx^\top \right] (w^*_{\textrm{aux}} - w^*_m) \\
       & = \left( \frac{1}{1+\Lambda} \right)^2 \| \Sigma_x^{\frac{1}{2}} (w^*_{\textrm{aux}} - w^*_m)\|^2 ,
   \end{align*}
   where $w^*_{\mathrm{aux}} = \frac{1}{\Lambda} \sum_{k=1}^K \lambda_k w^*_k$. The last equality follows from $\mathbb E_x[xx^\top ] = \Sigma_x + \mathbb E_x[x] \mathbb E_x[x]^\top  = \Sigma_x$.
   
   We then compute the variance term. Let $\Delta = \hat w_D - \bar w$. The variance is
   \begin{align}
       \mathbb E_{D,x} \left[ (x^\top \Delta)^2 \right] = \mathbb E_D \left[ \Delta^\top  \mathbb E_{x|D} [xx^\top ] \Delta\right] = \mathbb E_D \left[ \Delta^\top  \Sigma_x \Delta\right] = \operatorname{Tr} \left[ \Sigma_x \mathbb E_D \left[ \Delta \Delta^\top \right] \right] , \label{eq:a.2}
   \end{align}
   where the last equality follows from $\mathbb E[v^\top Av] = \mathbb E \operatorname{Tr} [v^\top Av] = \mathbb E \operatorname{Tr} [Avv^\top ] = \operatorname{Tr} \mathbb E [Avv^\top ]$ for any random vector $v$ and matrix $A$. By \cref{eq:2.2,eq:a.1}, we have
   \begin{align*}
       \Delta = \hat w_D - \bar w = \frac{1}{1 + \Lambda} (X^\top X)^{-1} X^\top  \left( \epsilon_m + \sum_{k=1}^K \lambda_k \epsilon_k \right) .
   \end{align*}
   Hence, we have
   \begin{align}
       \mathbb E_D \left[ \Delta \Delta^\top  \right] &= \left( \frac{1}{1+\Lambda} \right)^2 \mathbb E_X \left[ (X^\top X)^{-1} X^\top  \mathbb E_{\epsilon_m, \epsilon_t | X} \left[ \epsilon_m^2 + \sum_{k=1}^K \lambda_k^2 \epsilon_k^2\right] X (X^\top X)^{-1}\right] \nonumber \\
       & = \left( \frac{1}{1+\Lambda} \right)^2 \left( \sigma_m^2 + \sum_{k=1}^K \lambda_k^2 \sigma_k^2\right) \mathbb E_X \left[ (X^\top X)^{-1}\right] . \label{eq:a.3}
   \end{align}
   In the first step, the expectations of the product terms between two noises are zero due to independence. The second step follows from $\mathbb E_{\epsilon_t | X} [\epsilon_t^2] = \sigma_t^2 + \mathbb E_{\epsilon_t | X} [\epsilon_t]^2 = \sigma_t^2$ for any $t \in [K] \cup \{m\}$. 
   
   It remains to compute $\mathbb E_X [(X^\top X)^{-1}]$. Given that each row of $X \in \mathbb R^{N \times d}$ follows $\mathcal N(0, \Sigma_x)$, $X^\top X \in \mathbb R^{d \times d}$ follows a Wishart distribution, denoted as $\mathcal W_d(\Sigma_x, N)$, and $(X^\top X)^{-1}$ follows an inverse-Wishart distribution, denoted as $\mathcal W^{-1}_d(\Sigma_x^{-1}, N)$. By standard results~\citep{mardia2024multivariate}, we have, when $N > d + 1$,
   \begin{align}
       \mathbb E_X \left[ (X^\top X)^{-1} \right] = \frac{\Sigma_x^{-1}}{N-d-1} . \label{eq:a.4}
   \end{align}
   Finally, by \cref{eq:a.2,eq:a.3,eq:a.4}, we have the variance as
   \begin{align*}
       \mathbb E_{D,x}\left[ (x^\top (\hat w_D - \bar w))^2\right] & = \frac{\sigma_m^2 + \sum_{k=1}^K \lambda_k^2 \sigma_k^2}{(1+\Lambda)^2} \operatorname{Tr} \left[ \Sigma_x \cdot \frac{\Sigma_x^{-1}}{N-d-1}\right] \\
       & = \frac{\sigma_m^2 + \sum_{k=1}^K \lambda_k^2 \sigma_k^2}{(1+\Lambda)^2} \frac{d}{N-d-1} ,
   \end{align*}
   where the second step follows from $\operatorname{Tr}[\Sigma_x\Sigma_x^{-1}] = \operatorname{Tr}[I_d] = d$. This finishes the proof.
\end{proof}

\subsection{Proof of \cref{th:3.3}} \label{sec:ap-a2}

We apply the following result on the consistency of M-estimators. The original statement applies to maximization problems, but directly extends to minimization problems by flipping the signs.

\begin{proposition}[restated from \cite{van2000asymptotic}] \label{pp:a.1}
    Let $M_n$ be a random function of $\theta$ and $M$ be a deterministic function of $\theta$. Any estimator $\hat \theta_n$ converges in probability to $\theta_0$, i.e., $\hat \theta_n \overset{p}{\rightarrow} \theta_0$, if the following conditions hold:
    
    (C1) $\underset{\theta \in \Theta}{\sup} |M_n(\theta) - M(\theta)| \overset{p}{\rightarrow} 0$;

    (C2) For any $\epsilon > 0$, $\underset{\theta, d(\theta, \theta_0) \geq \epsilon}{\inf} M(\theta) > M(\theta_0)$;

    (C3) $M_n(\hat \theta_n) \leq M_n(\theta_0) + o_p(1)$.
\end{proposition}

We now prove \cref{th:3.3}.

{\renewcommand{\proofname}{Proof of \cref{th:3.3}}
\begin{proof}
    Let $\bm \mu^*$ be the vector of unknown distribution parameters and $\hat{\bm \mu}$ be the vector of the corresponding consistent estimators. We use $E_K(\bm \lambda;\bm \mu)$ for $E_K$ to clarify the parameters. By the boundedness assumption, we can define a compact set $\mathcal L_B = \{\bm \lambda, \|\bm \lambda\| \leq B\}$ such that
    \begin{align}
        \bm \lambda^* = \underset{\bm \lambda \in \mathcal L_B}{\arg \min}\ E_K(\bm \lambda; \bm \mu^*), \qquad \hat{\bm \lambda} = \underset{\bm \lambda \in \mathcal L_B}{\arg \min}\ E_K(\bm \lambda; \hat{\bm \mu}) . \label{eq:a.5}
    \end{align}
    Let us define $M_n(\bm \lambda) = E_K(\bm \lambda; \hat{\bm \mu})$, $M(\bm \lambda) = E_K(\bm \lambda; \bm \mu^*)$, $\hat \theta_n = \hat{\bm \lambda}$, and $\theta_0 = \bm \lambda^*$. By \cref{pp:a.1}, to prove $\hat{\bm \lambda} \overset{p}{\rightarrow} \bm \lambda^*$, it suffices to check the three conditions:

    (C3) holds with $o_p(1) = 0$. Specifically, by the definition of $\hat{\bm \lambda}$ in \cref{eq:a.5}, we have $M_n(\hat{\bm \lambda}) = E_K(\hat{\bm \lambda}; \hat{\bm \mu}) \leq E_K(\bm \lambda^*; \hat{\bm \mu}) = M_n(\bm \lambda^*)$.

    (C2) holds given that $\bm \lambda^*$ is the unique minimizer, $\mathcal L_B$ is compact, and $M(\bm \lambda)$ is continuous.

    (C1) holds given that $\hat{\bm \mu} \overset{p}{\rightarrow} \bm \mu^*$, $\mathcal L_B$ is compact, and $E_K(\bm \lambda; \bm \mu)$ is continuous over both $\bm \lambda$ and $\bm \mu$. 

    This completes the proof.
\end{proof}}

\section{Omitted Items in \cref{sec:4}}

\subsection{Proof of \cref{co:4.3}} \label{sec:ap-b1}

Our statement is a corollary of Theorem D.1 in \cite{tran2025spectral}. Here, we present it using notations specified in the subproblem setup in \cref{sec:4.2.1}.

\begin{proposition}[restated from \cite{tran2025spectral}, Theorem D.1] \label{pp:c.1}
    Let $S$ and $Z$ be $n \times n$ real symmetric matrices. Let $\lambda_k$ be the $k$-th largest eigenvalue of $S$ and $\sigma_k$ be the $k$-th largest singular value of $S$. Define index $l$, $1 \leq l \leq q$, such that the top $q$ singular values of $S$ correspond to its eigenvalues $\{\lambda_1, \ldots, \lambda_l > 0 \geq \lambda_{n-(q-l)+1}, \ldots, \lambda_n\}$. Define $\delta_i = \lambda_i - \lambda_{i+1}$ for $i \in [n-1]$. If $4\|Z\| \leq \min \{ \delta_l, \delta_{n-(q-l)}\}$ and $2\|Z\| < \sigma_q - \sigma_{q+1}$, then
    \begin{align}
        \|[M]_q - [S]_q\| \leq 6\|Z\| \left( \frac{\lambda_l}{\delta_l} + \log \frac{6\sigma_1}{\delta_{l}} + \frac{|\lambda_l|}{\delta_{n-(q-l)}} + \log \frac{6\sigma_1}{\delta_{n-(q-l)}} \right) . \label{eq:c.1}
    \end{align}
\end{proposition}
\begin{remark}
    To help understand the index $l$, consider a positive semi-definite $S$, whose eigenvalues are non-negative. Then, its top $q$ singular values correspond to its top $q$ eigenvalues, giving $l = q$.
\end{remark}

We then prove \cref{co:4.3}.

{\renewcommand{\proofname}{Proof of \cref{co:4.3}}
\begin{proof}
    Despite distinct proof techniques, the condition in \cref{pp:c.1} takes a similar form to our \cref{th:4.2}, both requiring a sufficient signal-to-noise ratio. To align them, we introduce the variable $r$ as $\min \{\frac{\delta_l(S)}{\|Z\|}, \frac{\delta_{n-(q-l)}(S)}{\|Z\|}\} \geq r \geq 4$ and embed it into the bound:
    \begin{align}
        \|[M]_q - [S]_q\| \leq \frac{3r}{r-2}\|Z\| \left( \frac{\lambda_l}{\delta_l} + \log \frac{6\sigma_1}{\delta_{l}} + \frac{|\lambda_l|}{\delta_{n-(q-l)}} + \log \frac{6\sigma_1}{\delta_{n-(q-l)}} \right) . \label{eq:c.2}
    \end{align}
    When $r=4$, this recovers \cref{eq:c.1}. We omit the detailed steps, which, in short, follow from altering Lemma 3.1 in that text. Next, we extend this to rectangular matrices using the Hermitian dilation technique~\citep{vershynin2018high}. Let $S$ and $Z$ be general $m \times n$ matrices. Define
    \begin{align*}
        \tilde S =  \begin{bmatrix}
                    0 & S \\
                    S^\top  & 0 \\
                    \end{bmatrix}, \qquad
        \tilde Z =  \begin{bmatrix}
                    0 & Z \\
                    Z^\top  & 0 \\
                    \end{bmatrix} , \qquad
        \tilde M = \tilde S + \tilde Z .
    \end{align*}
    Then $\tilde S, \tilde Z, \tilde M \in \mathbb R^{(m+n) \times (m+n)}$ are symmetric. The norm of the dilated matrix is the same as that of the original, e.g., $\|\tilde Z\| = \|Z\|$. Moreover, we have
    \begin{align*}
        [\tilde M]_{2q} - [\tilde S]_{2q} =  \begin{bmatrix}
                                        0 & [M]_q \\
                                        [M]_q^\top  & 0 \\
                                        \end{bmatrix}
                                    - \begin{bmatrix}
                                        0 & [S]_q \\
                                        [S]_q^\top  & 0 \\
                                        \end{bmatrix}
                                    = \begin{bmatrix}
                                        0 & [M]_q - [S]_q \\
                                        ([M]_q - [S]_q)^\top  & 0 \\
                                        \end{bmatrix} .
    \end{align*}
    Hence, $\|[\tilde M]_{2q} - [\tilde S]_{2q}\| = \|[M]_q - [S]_q\|$. Naturally, we apply \cref{eq:c.2} to $\|[\tilde M]_{2q} - [\tilde S]_{2q}\|$. Let us consider the index $l$: given the construction, the eigenvalues of $\tilde S$ are $\{\sigma_1(S), \ldots, \sigma_{\operatorname{rank}(S)}(S), 0,$ $\ldots, 0,  -\sigma_{\operatorname{rank}(S)}(S), \ldots, -\sigma_1(S)\}$; the top $2q$ singular values of $\tilde S$ correspond to its eigenvalues $\{\sigma_1(S), \ldots, \sigma_q(S) > 0 \geq -\sigma_q(S), \ldots, -\sigma_1(S)\}$, i.e., $l = q$. Hence, if $\min \{\frac{\delta_q(\tilde S)}{\|\tilde Z\|}, \frac{\delta_{(m+n)-q}(\tilde S)}{\|\tilde Z\|}\} \geq r \geq 4$ and $\frac{\sigma_{2q}(\tilde S) - \sigma_{2q+1}(\tilde S)}{\|Z\|} > 2$, we have
    \begin{align*}
        \|[\tilde M]_{2q} - [\tilde S]_{2q}\| \leq \frac{3r}{r-2} \|\tilde Z\| \left( \frac{\lambda_q(\tilde S)}{\delta_q(\tilde S)} + \log \frac{6\sigma_1(\tilde S)}{\delta_{q}(\tilde S)} + \frac{|\lambda_q(\tilde S)|}{\delta_{(m+n)-q}(\tilde S)} + \log \frac{6\sigma_1(\tilde S)}{\delta_{(m+n)-q}(\tilde S)} \right) .
    \end{align*}
    Substituting the norms and eigenvalues with those of $S,Z,M$ gives \cref{co:4.3}.
\end{proof}}

\subsection{Derivation of \cref{eq:4.6}} \label{sec:ap-b2}

First, we prove some useful lemmas on the singular values of the product of two matrices.

\begin{lemma} \label{lm:ap-b1}
    For any matrix A and vector v of compatible size, we have $\|Av\| \geq \sigma_{\min}(A)\|v\|$.
\end{lemma}
\begin{proof}
    Suppose the SVD of $A$ is $A = U\Sigma V^\top $, we have $\|Av\| = \|U\Sigma V^\top v\| = \|\Sigma V^\top v\| = \geq \sigma_{\min}(A)\|V^\top v\| = \sigma_{\min}(A)\|v\|$, where the second and last equalities follow from the fact that both $U$ and $V^\top $ are norm-preserving.
\end{proof}

\begin{lemma} \label{lm:ap-b2}
    For any matrices A and B of compatible size, we have $\sigma_{\min}(A)\sigma_i(B) \leq \sigma_i(AB) \leq \sigma_{\max}(A)\sigma_i(B)$ and $\sigma_i(A)\sigma_{\min}(B) \leq \sigma_i(AB) \leq \sigma_i(A)\sigma_{\max}(B)$.
\end{lemma}
\begin{proof}
    By the min-max characterization for singular values~\citep{bhatia2013matrix}, for any matrix $M$, we have $\sigma_i(A) = \underset{S, \dim(S)=i}{\max}\ \underset{x \in S, \|x\|=1}{\min} \|Ax\|$. Hence, we have the upper bound
    \begin{align*}
        \sigma_i(AB) & = \underset{S, \dim(S)=i}{\max}\ \underset{x \in S, \|x\|=1}{\min} \|ABx\| \\
        & \leq \underset{S, \dim(S)=i}{\max}\ \underset{x \in S, \|x\|=1}{\min} \|A\|\|Bx\| \\
        & = \|A\| \underset{S, \dim(S)=i}{\max}\ \underset{x \in S, \|x\|=1}{\min} \|Bx\| \\
        & = \sigma_{\max}(A) \sigma_i(B) .
    \end{align*}
    Similarly, we prove the lower bound, where the inequality follows from \cref{lm:ap-b1}.
    \begin{align*}
        \sigma_i(AB) & = \underset{S, \dim(S)=i}{\max}\ \underset{x \in S, \|x\|=1}{\min} \|ABx\| \\
        & \geq \underset{S, \dim(S)=i}{\max}\ \underset{x \in S, \|x\|=1}{\min} \sigma_{\min}(A) \| Bx\| \\
        & = \sigma_{\min}(A) \underset{S, \dim(S)=i}{\max}\ \underset{x \in S, \|x\|=1}{\min} \|Bx\| \\
        & = \sigma_{\min}(A)\sigma_i(B) .
    \end{align*}
    To prove the symmetric version, apply the established inequality to $(AB)^\top $: $\sigma_{\min}(B^\top ) \sigma_i(A^\top ) \leq \sigma_i(B^\top A^\top ) \leq \sigma_{\max}(B^\top )\sigma_i(A^\top )$. Since the singular values of a matrix and its transpose are the same, this gives the symmetric version.
\end{proof}

We now derive \cref{eq:4.6}.

{\renewcommand{\proofname}{Derivation of \cref{eq:4.6}}
\begin{proof}
By \cref{th:4.2} and the Cauchy–Schwarz inequality, we have
\begin{align*}
    \|([M]_q - [S]_q)e_m\|^2 \leq 3 \frac{1}{(r-1)^2}\|Se_m\|^2 + 3 \frac{1}{(r-1)^2}\|Ze_m\|^2 + 3\|P_UZe_m\|^2 .
\end{align*}
Plugging this into \cref{eq:4.3} and bounding $\|([S]_q - S)e_m\| \leq \sigma_{q+1}(S)$, we have
\begin{align}
    \mathbb E_D \|\hat w_m - w^*_m\|^2 & \leq \mathbb E_D \left[ \|X^{\dagger}\|^2 \left( 2\|([M]_q - [S]_q)e_m\|^2 + 2\|([S]_q-S)e_m\|^2\right) \right] \nonumber \\
    & \leq \mathbb E_D \left[ \|X^{\dagger}\|^2 \left( 2 \sigma_{q+1}^2(S) + 6 \frac{1}{(r-1)^2}\|Se_m\|^2 \right. \right. \nonumber \\
    & \qquad \qquad \left. \left. + 6 \frac{1}{(r-1)^2}\|Ze_m\|^2 + 6 \|P_UZe_m\|^2\right) \right] . \label{eq:b.1}
\end{align}
We wish to disentangle the randomness of $X$ embedded in $S = XW^*\Lambda$ from the constant values. By \cref{lm:ap-b2}, we have
\begin{align}
    \sigma_{q+1}(S) = \sigma_{q+1}(XW^*\Lambda) \leq \|X\|\sigma_{q+1}(W^*\Lambda) . \label{eq:b.2}
\end{align}
Next, given that $Se_m = XW^*\Lambda e_m = XW^*e_m = Xw^*_m$, we have $\|Se_m\| \leq \|X\| \|w^*_m\|$. Plugging this and \cref{eq:b.2} into \cref{eq:b.1} gives \cref{eq:4.6}.
\end{proof}}

\subsection{Proofs of \cref{lm:4.4,lm:4.5}} \label{sec:ap-b3}

We build on the following results of upper bounds on (sub)Gaussian random matrices. 

\begin{proposition}[\cite{vershynin2018high}] \label{pp:b.1}
    Let $A$ be an $m \times n$ random matrix with independent, mean-zero, subGaussian rows or columns. Then, we have
    \begin{gather}
        \mathbb E \|A\| \leq C K (\sqrt{m} + \sqrt{n}) , \nonumber \\
        \operatorname{Pr}(\|A\| \geq C K (\sqrt{m} + \sqrt{n} + t)) \leq 2e^{-t^2}, \quad \forall t > 0. \label{eq:b.4}
    \end{gather}
    Here, $C$ is a constant and $K = \underset{i}{\max}\|A_i\|_{\psi_2}$, where $\|A_{i}\|_{\psi_2}$ is the subGaussian norm of the $i$-th row/column. In particular, when $A_i \sim \mathcal N(0, \Sigma)$ for any $i$, we have $K = \sqrt{\|\Sigma\|}$ and $C=1$.
\end{proposition}

This enables us to prove \cref{lm:4.4}.

{\renewcommand{\proofname}{Proof of \cref{lm:4.4}}
\begin{proof}
    The bounds for $\|Ze_j\|$ directly follow from \cref{pp:b.1}. Now consider $\|P_UZe_j\|$. Recall that $P_U = U_{:,q}U_{:,q}^\top $, where $U_{:,q} \in \mathbb R^{m \times q}$ is norm-preserving, giving
    \begin{align*}
        \|P_UZe_j\| = \|U_{:,q}U_{:,q}^\top Ze_j\| = \|U_{:,q}^\top Ze_j\| .
    \end{align*}
    Since $Ze_j$ is subGaussian with mean zero and subGaussian norm $\sigma$, and $U_{:,q}^\top $ is fixed and independent of $Z$, we know that $U_{:,q}^\top Ze_j \in \mathbb R^q$ is subGaussian with mean zero and subGaussian norm
    \begin{align*}
        \|U_{:,q}^\top Ze_j\|_{\psi_2} = \underset{\substack{v \in \mathbb R^q,\\\|v\| = 1}}{\sup} \langle U_{:,q}^\top Ze_j, v\rangle = \underset{\substack{v \in \mathbb R^q,\\\|v\| = 1}}{\sup} \langle Ze_j, U_{:,q}v\rangle \leq \underset{\substack{v \in \mathbb R^q,\\\|v\| = 1}}{\sup} \langle Ze_j, v\rangle = \sigma ,
    \end{align*}
    where the inequality follows from $\|U_{:,q}v\| = \|v\|$ and $U_{:,q}v \in \operatorname{col}(U) \subseteq \mathbb R^q$. Thereby, \cref{pp:b.1} is applicable to $P_UZe_j$, which gives the desired bounds.
\end{proof}}

To prove \cref{lm:4.5}, we establish a useful lemma for bounding the $p$-th moment of the norm of a subGaussian random matrix, based on the tail bound in \cref{eq:b.4}.

\begin{lemma} \label{lm:ap-b3}
    Let A be an $m \times n$ random matrix with independent, mean-zero, subGaussian rows or columns. Then, for any positive integer $p$, we have
    \begin{align*}
        \mathbb E \|A\|^p \leq CK^p(\sqrt{m} + \sqrt{n})^p .
    \end{align*}
    Here, $C$ is a constant dependent on $p$ and $K = \underset{i}{\max}\|A_i\|_{\psi_2}$. If $A_i \sim \mathcal N(0, \Sigma)$, $K = \sqrt{\|\Sigma\|}$.
\end{lemma}
\begin{proof}
    Let $c_1 = CK$ and $c_2 = CK(\sqrt{m} + \sqrt{n})$, then \cref{eq:b.4} can be rewritten as, 
    \begin{align}
        \operatorname{Pr}(\|A\| \geq c_2 + c_1t) \leq 2e^{-t^2}, \quad \forall t > 0 . \label{eq:b.7}
    \end{align}
    By the tail integral formula, we have
    \begin{align}
        \mathbb E \|A\|^p & = p \int_0^\infty t^{p-1} \operatorname{Pr}(\|A\| > t) dt \nonumber \\
        & = p \int_{0}^{c_2} t^{p-1} \operatorname{Pr}(\|A\| > t)dt + p \int_{c_2}^\infty t^{p-1} \operatorname{Pr}(\|A\| > t)dt . \label{eq:b.5}
    \end{align}
    For the first split, we have
    \begin{align}
        p \int_{0}^{c_2}t^{p-1} \Pr(\|A\| > t)dt & \leq p \int_{0}^{c_2}t^{p-1} \cdot 1 dt = c_2^p . \label{eq:b.6}
    \end{align}
    For the second split, we do a change-of-variable $t = c_2 + c_1u$ and plug in \cref{eq:b.7}:
    \begin{align}
        p \int_{c_2}^\infty t^{p-1} \operatorname{Pr}(\|A\| > t)dt & = c_1 p \int_{0}^\infty (c_2 + c_1u)^{p-1} \operatorname{Pr}(\|A\| > c_2 + c_1u) du \nonumber \\
        & \leq c_1p \int_0^\infty (c_2 + c_1u)^{p-1} 2e^{-u^2} du . \label{eq:b.8}
    \end{align}
    Next, we deal with the integral $\int_0^{\infty} (c_2 + c_1u)^{p-1}2e^{-u^2}du \leq Cc_2^{p-1}$. With the binomial expansion for $(c_2 + c_1u)^{p-1}$, we have
    \begin{align}
        \int_0^{\infty} (c_2 + c_1u)^{p-1}2e^{-u^2}du & = \int_0^\infty \sum_{k=1}^{p-1} \binom{p-1}{k} c_2^{p-1-k} (c_1u)^{k} \cdot 2e^{-u^2}du \nonumber \\
        & = \sum_{k=1}^{p-1} \binom{p-1}{k} c_2^{p-1-k} c_1^{k} \int_{0}^{\infty} 2u^ke^{-u^2}du \nonumber \\
        & \leq c_2^{p-1} \sum_{k=1}^{p-1} \binom{p-1}{k} \int_{0}^{\infty} 2u^ke^{-u^2}du \qquad (\because c_1 \leq c_2) \nonumber \\
        & =c_2^{p-1} \sum_{k=1}^{p-1} \binom{p-1}{k} \Gamma \left( \frac{k+1}{2} \right), \label{eq:b.9}
    \end{align}
    where the last step follows from the Gamma integral. The constant $\sum_{k=1}^{p-1} \binom{p-1}{k} \Gamma \left( \frac{k+1}{2} \right)$ is finite and depends on $p$ only. Plugging \cref{eq:b.9} into \cref{eq:b.8}, we bound the second split:
    \begin{align}
        p \int_{c_2}^\infty t^{p-1} \operatorname{Pr}(\|A\| > t)dt \leq c_1p \cdot Cc_2^{p-1} \leq Cc_2^p . \label{eq:b.10}
    \end{align}
    Finally, combining \cref{eq:b.5,eq:b.6,eq:b.10}, we have
    \begin{align*}
        \mathbb E \|A\|^p \leq c_2^p + Cc_2^p \leq Cc_2^p = C K^p (\sqrt{m} + \sqrt{n})^p ,
    \end{align*}
    which completes the proof.
\end{proof}

We are now ready to prove \cref{lm:4.5}.

{\renewcommand{\proofname}{Proof of \cref{lm:4.5}}
\begin{proof}
    We first deal with $Ze_m = P_X\epsilon\Lambda e_m = P_X\epsilon_m$. Since $P_X = X(X^\top X)^{-1}X^\top $ is a rank-$d$ projector, it can be written as $P_X = QQ^\top $, where $Q \in \mathbb R^{N \times d}$ has orthonormal columns. Hence, $Q$ is norm-preserving, and we have
    \begin{align}
        \|Ze_m\| = \|P_X \epsilon_m\| = \|QQ^\top  \epsilon_m\| = \|Q^\top  \epsilon_m\|
    \end{align}
    Consider the distribution of $Q^\top \epsilon_m$ conditioned on $X$. Since $\epsilon_m \sim \mathcal N(0, \sigma_m^2 I_N)$, and $Q^\top  \mid X$ is fixed and independent of $e_m$, we know that $Q^\top  \epsilon_m \mid X$ follows a Gaussian distribution with mean zero and covariance $\sigma_m^2 Q^\top Q = \sigma_m^2 I_d$. Applying \cref{lm:ap-b3}, we have
    \begin{align*}
        \mathbb E_{\epsilon | X}\|Ze_m\|^2 = \mathbb E_{\epsilon | X} \|Q^\top \epsilon_m\|^2 \leq C \sigma_m^2 (\sqrt{d} + 1)^2 \leq C_1 \sigma_m^2 d .
    \end{align*}
    Next, consider $P_UZe_m = P_UP_X \epsilon_m$. Recall that $P_U = U_{:,q}U_{:,q}^\top $, where $U_{:,q}$ denotes the first $q$ left singular vectors of $S = XW^*\Lambda$. We have, as $U_{:,q}$ is norm-preserving,
    \begin{align*}
        \|P_UZe_m\| = \|U_{:,q}U_{:,q}^\top P_X\epsilon_m\| = \|U_{:,q}^\top  P_X \epsilon_m\| .
    \end{align*}
    Moreover, since both $U_{:,q}$ and $P_X$ are fixed given $X$, $U_{:,q}^\top P_X\epsilon \mid X$ is Gaussian with mean zero and covariance of norm
    \begin{align}
        \sigma_m^2 \|(U_{:,q}^\top P_X)(U_{:,q}^\top P_X)^\top \| = \sigma_m^2 \|U_{:,q}^\top  P_X U_{:,q}\| \leq \sigma_m^2 \|U_{:,q}^\top \| \|P_X\| \|U_{:,q}\| \leq \sigma_m^2 .
    \end{align}
    Also, note that $U_{:,q}^\top  P_X \epsilon$ is of size $q$. Applying \cref{lm:ap-b3}, we have
    \begin{align*}
        \mathbb E_{\epsilon | X} \|P_UZe_m\|^2 = \mathbb E_{\epsilon | X} \|U_{:,q}^\top  P_X\epsilon_m\|^2 \leq C \sigma_m^2 (\sqrt{q} +1)^2 \leq C_2 \sigma_m^2q .
    \end{align*}
    This finishes the proof.
\end{proof}}

\subsection{Proof of \cref{lm:4.6}} \label{sec:ap-b4}

We build on the following result, which is derived from \cite{chen2005condition}, Lemma 4.1 on the smallest singular value of a Gaussian matrix. We note that the original statement assumes wide matrices, but it directly extends to tall ones by taking the transpose. 

\begin{proposition}[\cite{halko2011finding}, Proposition A.3] \label{pp:ap-b2}
    Let G be an $n \times m$ random matrix with $n \geq m \geq 2$ and i.i.d. entries drawn from $\mathcal N(0,1)$. For any $t > 0$, we have
    \begin{align*}
        \operatorname{Pr}(\|G^{\dagger}\| > t) \leq \frac{1}{\sqrt{2\pi(n-m+1)}} \left( \frac{e\sqrt{n}}{n-m+1} \right)^{n-m+1}t^{-(n-m+1)} .
    \end{align*}
\end{proposition}

Again, we use this tail probability to establish expectation bounds on the $p$-th moment of $\|G^{\dagger}\|$.

\begin{lemma} \label{lm:ap-b4}
    Let G be an $n \times m$ random matrix with $n \geq m \geq 2$ and i.i.d. entries drawn from $\mathcal N(0,1)$. When $n > m + p - 1$, we have
    \begin{align*}
        \mathbb E \|G^{\dagger}\|^p \leq \left( \frac{e\sqrt{n}}{n-m-p+1} \right)^p .
    \end{align*}
\end{lemma}
\begin{proof}
    By the tail integral formula, we have
    \begin{align}
        \mathbb E \|G^{\dagger}\|^p & = p \int_0^\infty t^{p-1} \operatorname{Pr}(\|G^{\dagger}\| > t) dt \nonumber \\
        & = p \int_{0}^{E} t^{p-1} \operatorname{Pr}(\|G^{\dagger}\| > t)dt + p \int_{E}^\infty t^{p-1} \operatorname{Pr}(\|G^{\dagger}\| > t)dt ,\label{eq:b.11}
    \end{align}
    where $E$ is some constant to be optimized later. For the first split, we have
    \begin{align}
        p \int_{0}^{E}t^{p-1} \Pr(\|G^{\dagger}\| > t)dt & \leq p \int_{0}^{E}t^{p-1} \cdot 1 dt = E^p . \label{eq:b.12}
    \end{align}
    For the second split, we utilize \cref{pp:ap-b2}. First, define $l = n-m+1$ and $C = \frac{1}{\sqrt{2\pi l}}(\frac{e\sqrt{n}}{l})^l$. We have
    \begin{align}
        p \int_{E}^\infty t^{p-1} \operatorname{Pr}(\|G^{\dagger}\| > t)dt & \leq p \int_E^\infty t^{p-1} \cdot C t^{-l} dt \nonumber \\
        & = Cp \int_E^\infty t^{p-1-l} dt \nonumber \\
        & = Cp \frac{1}{l-p}E^{-(l-p)}, \qquad \text{given that } l-p > 0 . \label{eq:b.13}
    \end{align}
    Plugging \cref{eq:b.12,eq:b.13} into \cref{eq:b.11} gives
    \begin{align*}
        \mathbb E \|G^{\dagger}\|^p \leq E^p + C \frac{p}{l-p}E^{-(l-p)} , \qquad \text{given that } l > p .
    \end{align*}
    Minimizing the RHS over $E$ gives $E^* = C^{1/l}$. Plugging in $E^*$ and $C$, we have
    \begin{align*}
        \mathbb E \|G^{\dagger}\|^p \leq \frac{l}{l-p} C^{\frac{p}{l}} & = \frac{l}{l-p} \left( \frac{1}{\sqrt{2 \pi l}}\right)^{\frac{p}{l}} \left( \frac{e\sqrt{n}}{l}\right)^p \nonumber \\
        & \leq \frac{l}{l-p} \left( \frac{e\sqrt{n}}{l} \right)^p \qquad (\because \left( \frac{1}{\sqrt{2\pi l}}\right)^{\frac{p}{l}} < 1) \\
        & = \frac{1}{l-p} \left( \frac{1}{l} \right)^{p-1} \left(e\sqrt{n}\right)^p \\
        & \leq \left( \frac{1}{l-p}\right)^p \left( e\sqrt{n} \right)^p . \qquad (\because  \frac{1}{l} \leq \frac{1}{l-p})
    \end{align*}
    Substituting $l = n-m+1$ gives the desired result.
\end{proof}

We now prove \cref{lm:4.6}.

{\renewcommand{\proofname}{Proof of \cref{lm:4.6}}
\begin{proof}
    By the setups in \cref{sec:2}, each row of $X$ is drawn independently from $\mathcal N(0, \Sigma_x)$. Hence, $X$ can be written as $X = G \Sigma_x^{1/2}$, where $G$ has i.i.d. standard Gaussian entries. We have $\|X^{\dagger}\| = 1/\sigma_{\min}(X)$. By \cref{lm:ap-b2}, we have
    \begin{gather*}
    \sigma_{\max}(X) = \sigma_{\max}(G\Sigma_x^{1/2}) \leq \sigma_{\max}(G) \sigma_{\max}(\Sigma_x^{1/2}), \\
    \sigma_{\min}(X) = \sigma_{\min}(G\Sigma_x^{1/2}) \geq \sigma_{\min}(G) \sigma_{\min}(\Sigma_x^{1/2}) .
    \end{gather*}
    Hence, we can convert the problem of bounding $X$ to bounding $G$: given that $\sigma_{\min}(\Sigma_x) > 0$,
    \begin{align}
        \mathbb E_X \|X^{\dagger}\|^2 & = \mathbb E_X \left[ \frac{1}{\sigma_{\min}^2(X)} \right] \nonumber \\
        & \leq \frac{1}{\sigma_{\min}(\Sigma_x)} \mathbb E \left[ \frac{1}{\sigma_{\min}^2(G)}\right] = \frac{1}{\sigma_{\min}(\Sigma_x)} \mathbb E \|G^{\dagger}\|^2, \label{eq:b.14} \\
        \mathbb E_X \left[ \|X^{\dagger}\|^2 \|X\|^2 \right] & = \mathbb E_X \left[ \frac{\sigma_{\max}^2(X)}{\sigma_{\min}^2(X)} \right] \nonumber\\
        & \leq \frac{\sigma_{\max}(\Sigma_x)}{\sigma_{\min}(\Sigma_x)} \mathbb E \left[ \frac{\sigma_{\max}^2(G)}{\sigma_{\min}^2(G)} \right] = \kappa(\Sigma_x) \mathbb E \left[ \|G^{\dagger}\|^2\|G\|^2\right] \label{eq:b.15}
    \end{align}
    The bound on $\mathbb E \|G^{\dagger}\|^2$ directly follows from \cref{lm:ap-b4}:
    \begin{align*}
        \mathbb E \|G^{\dagger}\|^2 \leq \left( \frac{e\sqrt{N}}{N-d-1} \right)^2 = C_3 \left( \frac{1}{\sqrt{N} - \frac{d+1}{\sqrt{N}}} \right)^2 .
    \end{align*}
    Plugging this into \cref{eq:b.14} gives the bound on $\mathbb E_X \|X^{\dagger}\|^2$. 
    
    By the Cauchy-Schwarz inequality, we have
    \begin{align}
        \mathbb E \left[ \|G^{\dagger}\|^2\|G\|^2\right] \leq \sqrt{\mathbb E \|G^{\dagger}\|^4 \cdot \mathbb E \|G\|^4}. \label{eq:b.16}
    \end{align}
    By \cref{lm:ap-b4}, we have
    \begin{align}
        \mathbb E \|G^{\dagger}\|^4 \leq \left( \frac{e\sqrt{N}}{N-d-3} \right)^4 = C \left( \frac{1}{\sqrt{N} - \frac{d+3}{\sqrt{N}}} \right)^4 . \label{eq:b.17}
    \end{align}
    By \cref{lm:ap-b3}, we have
    \begin{align}
        \mathbb E \|G\|^4 \leq C(\sqrt{N} + \sqrt{d})^4 . \label{eq:b.18}
    \end{align}
    Combining \cref{eq:b.15,eq:b.16,eq:b.17,eq:b.18}, we have
    \begin{align*}
        \mathbb E_X \left[ \|X^{\dagger}\|^2 \|X\|^2\right] \leq C_4 \kappa(\Sigma_x) \left( \frac{\sqrt{N} + \sqrt{d}}{\sqrt{N} - \frac{d+3}{\sqrt{N}}} \right)^2 ,
    \end{align*}
    which concludes the proof.
\end{proof}}

\newpage
\section{Experimental Details} \label{sec:ap-c}

\paragraph{Setup.} For simplicity, we adopt an isotropic data assumption with $\Sigma_x = I_d$. Following our setup in \cref{sec:2}, we sample $N$ training features from $\mathcal N(0, I_d)$ and stack them into $X \in \mathbb R^{N \times d}$. We sample the ground-truth model for the main task $w^*_m$ from $\mathcal N(0, I_d)$ and amplify it by $5$ for a reasonable signal-to-noise ratio. The auxiliary ground-truth models $w^*_k, k \in [K]$ are generated by adding Gaussian noise to $w^*_m$, where the variance of the noise distribution controls the similarity, i.e., $w^*_k = w^*_m + \tilde \epsilon_k$, where $\tilde \epsilon_k \sim \mathcal N(0, \tilde \sigma_k^2)$. Finally, we create labels with the label noise: for any $t \in [K] \cup \{m\}$, set $Y_t = Xw^*_t + \epsilon_t$, where $\epsilon_t \overset{i.i.d.}{\sim} \mathcal N(0, \sigma_t^2)$. 

For each weight combination $\{\lambda_k\}_{k=1}^K$ considered, we fit a shared model $\hat w_D$ by \cref{eq:2.2} for linear regression and a task-specific estimator $\hat w_m$ by \cref{pp:2.2} for linear neural networks. Then, we compute their MSEs on $10000$ newly sampled test data, where this amount is chosen to be large to best simulate the expectation over $(x,y_m)$ in \cref{eq:2.5}. To simulate the expectation over $D$, we repeat over multiple draws of $D$ for each set of parameters and report the average.

\paragraph{Parameters.} For \cref{fig:1a}, we set $N=100$, $d=80$, $\tilde \sigma_1 = 5$, $\sigma_m = 10$, and $\sigma_1 = 1$. \cref{fig:2} plots the result for two auxiliary tasks with the same parameters plus $\tilde \sigma_2 = \sigma_2 = 5$. We repeat each set of parameters 100 times and report the average.

For the linear neural network setting, we use $K=1$, $q=1$, $N=1000$, $d=10$, $\tilde \sigma_1 = 0.1$, $\sigma_m = 10$, and $\sigma_1 = 1$ for a proper scale. We repeat each set of parameters 10 times and report the average. The number is reduced due to the computational costs of performing SVDs on high-dimensional matrices. For simplicity, in \cref{fig:1c}, we directly use the formations in the learning context to simulate the matrices $M, S, Z$.

\begin{figure*}[h]
  \centering
  \includegraphics[width=0.5\textwidth]{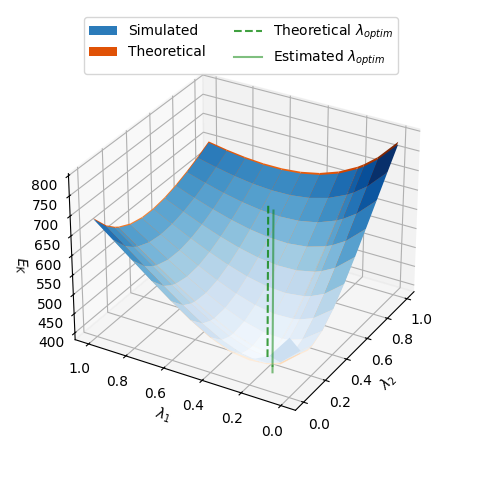}
  \caption{Linear regression with $K=2$.}
  \label{fig:2}
\end{figure*}

\end{document}